\documentclass[twoside]{article}

\usepackage[accepted]{aistats2026}
\usepackage[round]{natbib}

\usepackage{hyperref}
\usepackage[capitalize,noabbrev]{cleveref}
\usepackage{mathtools} 
\usepackage{amsmath}
\usepackage{amsfonts}
\usepackage{amssymb}
\usepackage{subcaption}
\usepackage{graphicx}
\usepackage{multirow}
\usepackage{booktabs}
\usepackage{float}
\usepackage{booktabs} 
\usepackage{tikz} 
\usepackage{subcaption} 
\usepackage{booktabs} 
\usepackage{colortbl} 
\usepackage{makecell} 
\usepackage{stfloats}   
\usepackage{tabularx}   

\newcommand{\paren}[1]{\mathopen{}\left( {#1}_{{}_{}}\,\negthickspace\right)\mathclose{}}
\newcommand{\bracket}[1]{\mathopen{}\left[ {#1}_{{}_{}}\,\negthickspace\right]\mathclose{}}
\newcommand{\fbracket}[1]{\mathopen{}\left\{ {#1}_{{}_{}}\,\negthickspace\right\}\mathclose{}}
\newcommand{\euclidean}[1]{\mathbb{R}^{#1}}

\newcommand{\Gram}{\operatorname{Gram}}
\usepackage{amsthm}
\newtheorem{theorem}{Theorem}[section]
\usepackage{thmtools, thm-restate}

\newtheorem{definition}{Definition}[section]

\newtheorem{lemma}[theorem]{Lemma}
\newtheorem{remark}{Remark}[section]

\begin{document}

%

%

\twocolumn[
\aistatstitle{Learning Equivariant Functions via Quadratic Forms}

\aistatsauthor{
Pavan~Karjol$^{1}$ \And 
Vivek~V~Kashyap$^{1}$ \And 
Rohan~Kashyap$^{2}$ \And 
Prathosh~A~P$^{1}$
}

\vspace{0.2cm} 

\aistatsaddress{
$^{1}$Department of Electrical Communication Engineering, Indian Institute of Science, Bengaluru, Karnataka\\
$^{2}$Computer Science, Carnegie Mellon University, Pittsburgh, Pennsylvania, USA
}
]

\begin{abstract}
In this study, we introduce a method for learning group (known or unknown) equivariant functions by learning the associated quadratic form $x^T A x$ corresponding to the group from the data. Certain groups, known as orthogonal groups, preserve a specific quadratic form, and we leverage this property to uncover the underlying symmetry group under the assumption that it is orthogonal. By utilizing the corresponding unique symmetric matrix and its inherent diagonal form, we incorporate suitable inductive biases into the neural network architecture, leading to models that are both simplified and efficient. Our approach results in an invariant model that preserves norms, while the equivariant model is represented as a product of a norm-invariant model and a scale-invariant model, where the “product” refers to the group action.

Moreover, we extend our framework to a more general setting where the function acts on tuples of input vectors via a diagonal (or product) group action. In this extension, the equivariant function is decomposed into an angular component extracted solely from the normalized first vector and a scale-invariant component that depends on the full Gram matrix of the tuple. This decomposition captures the inter-dependencies between multiple inputs while preserving the underlying group symmetry.

We assess the effectiveness of our framework across multiple tasks, including polynomial regression, top quark tagging, and moment of inertia matrix prediction. Comparative analysis with baseline methods demonstrates that our model consistently excels in both discovering the underlying symmetry and efficiently learning the corresponding equivariant function.
\end{abstract}

\section{Introduction}
\label{sec:intro}

Symmetry provides powerful inductive biases in deep learning, improving generalization and sample efficiency \citep{cohen2013learning, cohen2014learning, EMLPfinzi2021practical, dehmamy2021automatic}. When the underlying group is known, it can be encoded into architectures via invariant or equivariant operations, as in CNNs, which achieve translational equivariance through weight sharing. 

Two main strategies exist for incorporating symmetries: (i) data augmentation with parameterized distributions \citep{benton2020learning}, and (ii) architectures with intrinsic symmetry \citep{zaheer2017deep, kicki2020computationally, zhou2020meta}. In contrast, we develop a quadratic-form framework for learning $G$-equivariant functions, showing they canonically decompose into norm-invariant and scale-invariant terms.

Despite progress across images \citep{anselmi2019symmetry, kondor2018generalization}, graphs \citep{han2022equivariantgraphhierarchybasedneural}, and PDEs \citep{horie2023physicsembeddedneuralnetworksgraph}, challenges remain: most methods require prior group knowledge, rely on heavy computations (e.g., null-spaces, exponential maps), and scale poorly. Moreover, in many applications the group action is not directly observable, complicating symmetry exploitation.

We address these limitations by proposing a framework applicable with or without prior group knowledge. Focusing on orthogonal groups, we exploit their preservation of quadratic forms $x^\top A x$. For $O(n)$ this corresponds to Euclidean norm preservation; for Lorentz groups it corresponds to the Minkowski norm, fundamental in physics. This geometric perspective enables efficient equivariance and symmetry discovery. Section~\ref{Proposed Method} details the mathematical framework.

\subsection{Contributions}
\begin{itemize}
    \item We show that any function that is equivariant to a generalized orthogonal group (i.e., one that preserves a quadratic form $x^TAx$) can be expressed as the product of two invariant functions: a \textbf{scale-invariant} component that acts on the input normalized by the quadratic-form norm, and a \textbf{norm-invariant} component that depends only on that norm.  The product is interpreted in the context of the group’s action on the input space.
    \item We extend this decomposition to the setting with diagonal group action on the input product space. In this case, the equivariant map splits into (i) an \textbf{angular} term computed from only the normalized first input, and (ii) a \textbf{scale-invariant} term that depends on the full Gram matrix (all pairwise quadratic-form inner products) of the inputs, thereby capturing their inter-dependencies.
    \item Building on this decomposition, we introduce a \textbf{symmetry discovery} procedure: by learning the underlying quadratic form directly from data, the method uncovers the preserved symmetry without assuming it in advance.
    \item We validate the approach across diverse tasks: synthetic polynomial regression, top-quark tagging (Lorentz-invariant classification), and prediction of moment-of-inertia matrices (rotation-equivariant regression), demonstrating accurate symmetry recovery and strong predictive performance.
\end{itemize}

\section{Related Work}
\subsection{Group Equivariance}
Significant progress has been made both theoretically and practically in symmetry learning for deep learning architectures \citep{kondor2018clebsch, bekkers2021bsplinecnnsliegroups, EMLPfinzi2021practical, bronstein2021geometric, dehmamy2021automatic}. The classical example is the G-equivariant neural networks (G-CNNs) \citep{cohen2014learning} which extends CNN's to accommodate a broader class of symmetry groups beyond translations. Building on this, \citep{kondor2018generalization} developed convolution formulations for arbitrary compact group actions, while \citep{cohen2019general} further extended G-CNNs to homogeneous spaces through equivariant linear mappings.

Equivariant multi-layer perceptrons (EMLPs) for general groups were introduced by \citep{EMLPfinzi2021practical}, leveraging null space computations to achieve equivariance; however, scalability challenges persist, especially for higher-order groups and also requires explicit knowledge of the group to encode the group representations. Similarly, \citep{otto2023unified} proposed a linear-algebraic approach that enforces symmetries through linear constraints and convex penalties, though this approach also encounters scalability limitations. 

In related work, \citep{park2022learningsymmetricembeddingsequivariant} employed symmetric embedding networks to learn equivariant features for downstream tasks, although their method requires prior knowledge of the symmetry group. Contrastive learning methods were also developed by \citep{shakerinava2022learning, dangovski2022equivariantcontrastivelearning} for self-supervised equivariant representations. Symmetry learning has found practical applications in domains such as protein structure prediction, where molecular properties require invariance for 3D rotations \citep{yim2023se3diffusionmodelapplication}.

\subsection{Automatic Symmetry Discovery}

Beyond architectures with fixed symmetries, several works focus on \emph{automatic discovery} of symmetries directly from data. These approaches differ in how they parameterize transformations, enforce invariance or equivariance, and whether they yield a usable prediction model.

\subsubsection{Distribution over generators}  
Augerino~\citep{benton2020learning} learns a distribution over Lie algebra generators and samples transformations during training, applying an invariance or equivariance loss depending on the task. Similarly, LieGAN~\citep{yang2023liegan} learns such distributions but enforces consistency via an adversarial loss: a discriminator distinguishes whether transformed samples remain within the data distribution. A key difference is that Augerino yields a trained prediction model with soft equivariance, while LieGAN focuses on generator discovery and does not provide a predictor.

\subsubsection{Post-hoc generator extraction}  
LieGG~\citep{moskalev2022liegg} extracts Lie algebra generators from trained models using linear-algebraic analysis. While useful for diagnosing symmetries, the approach is inherently post-hoc and does not construct an explicit equivariant model.

\subsubsection{Vector-field approaches}  
Learning Infinitesimal Generators (LIG)~\citep{ko2024lig} models continuous flows via Neural ODEs, extending symmetry discovery beyond affine groups. However, enforcing equivariance would require synchronized flows in the output space, computationally non-trivial and generally infeasible. Symmetry Discovery Beyond Affine (SDBA)~\citep{shaw2024sdba} also identifies vector fields that annihilate functions, but is typically limited to single-generator settings and mainly yields invariants rather than full equivariant predictors. Both methods therefore provide only indirect or implicit symmetry information.

\subsubsection{Discrete-group inference}  
Bispectral Neural Networks (BNN)~\citep{sanborn2022bnn} address discrete groups by reconstructing Cayley tables from bispectrum features, thus inferring the group explicitly. While highly interpretable for finite groups, no analogous bispectral construction exists for continuous symmetries.

\subsubsection{Our distinction}  
In contrast, our framework simultaneously \emph{discovers} the underlying symmetry and provides an explicitly \emph{interpretable prediction model}. By learning the quadratic form preserved by an orthogonal group, we obtain a canonical decomposition into norm- and scale-invariant components. Unlike distribution-based methods (Augerino, LieGAN), which enforce equivariance only implicitly, and unlike post-hoc or indirect approaches (LieGG, LIG, SDBA), our method integrates discovery and prediction in a unified and interpretable manner, combining the clarity of BNN with applicability to continuous groups.

\section{Preliminaries}

\begin{definition}[\textbf{Group}]
A group \( (G,\cdot) \) is a set \(G\) with a binary operation \(\cdot:G\times G\to G\) satisfying:  
(i) Closure: \(a\cdot b\in G\) for all \(a,b\in G\);  
(ii) Associativity: \((a\cdot b)\cdot c=a\cdot(b\cdot c)\);  
(iii) Identity: \(\exists e\in G\) s.t.\ \(e\cdot a=a\cdot e=a\) for all \(a\in G\);  
(iv) Inverse: \(\forall a\in G, \ \exists a^{-1}\in G\) with \(a\cdot a^{-1}=a^{-1}\cdot a=e\).
\end{definition}

\begin{definition}[\textbf{Group Action}]
A (left) \emph{group action} of \(G\) on a set \(X\) is a map \(G\times X\to X,\ (g,x)\mapsto g\cdot x\), such that  
(i) \(e\cdot x=x\) for all \(x\in X\);  
(ii) \((gh)\cdot x=g\cdot(h\cdot x)\) for all \(g,h\in G,\ x\in X\).
\end{definition}

\begin{definition}[\textbf{Orbit}]
Given a group action of \(G\) on \(X\), the \emph{orbit} of \(x\in X\) is
\begin{equation}
G\cdot x=\{\,g\cdot x \mid g\in G\,\}.
\label{eq:def_orbit}
\end{equation}
\end{definition}

\begin{definition}[\textbf{General Linear Group}]
The \emph{general linear group} \(\mathrm{GL}(n,\mathbb{R})\) is the group of all invertible \(n\times n\) real matrices under matrix multiplication.
\end{definition}

\begin{definition}[\textbf{Equivariant Function}]
Let \(G\) act on sets \(X\) and \(Y\). A function \(f:X\to Y\) is \emph{equivariant} if
\[
f(g\cdot x)=g\cdot f(x), \quad \forall g\in G,\ x\in X.
\]
Equivariance means \(f\) commutes with the group action.  
An \emph{invariant function} is the special case where \(g\cdot f(x)=f(x)\) for all \(g\in G\).
\end{definition}

\begin{figure}[h]
    \centering
    \includegraphics[width=\columnwidth]{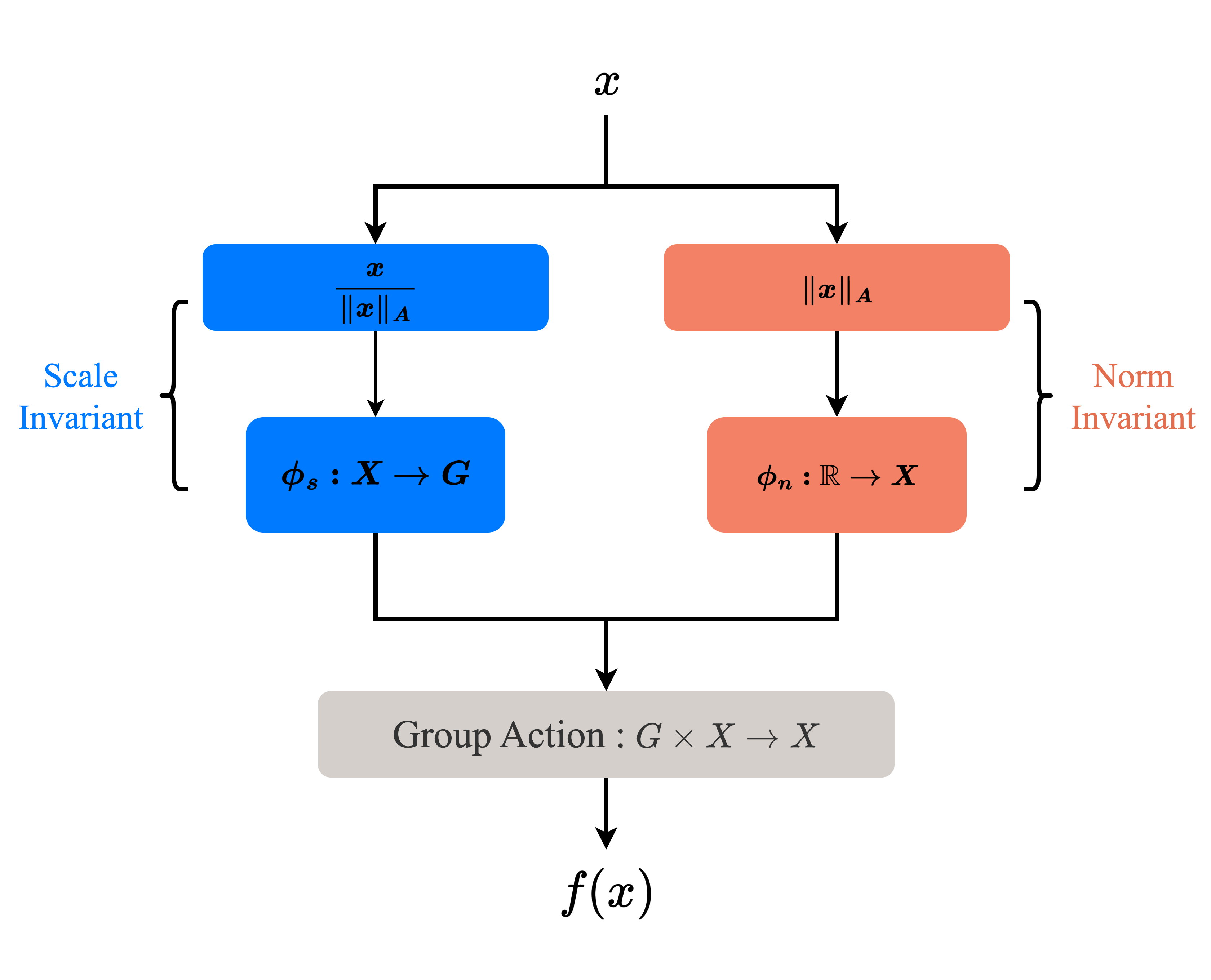} %
    \caption{Schematic representation of the proposed method. \( \phi_n \) and \( \phi_s \) denote neural networks, where the norm-invariant network takes input norms, and the scale-invariant network takes normalized inputs where \( \Vert x \Vert_A = \operatorname{sign}(x^T A x) \sqrt{\vert x^T A x \vert} \).}
    \label{fig:block_diagram_1}
\end{figure}

\section{Proposed Method}
\label{Proposed Method}
We present a general framework for learning group-equivariant functions by embedding appropriate quadratic forms within our network architecture. Specifically, we focus on groups that preserve quadratic forms of the type ``\( x^T A x \)'', known as orthogonal groups. Our framework supports two cases: when the underlying group is known and when it is unknown. In the first case, we fix the quadratic form within the architecture. In the second case, we allow the quadratic form to be learned. We begin with a brief overview of quadratic forms before presenting our main contributions.

\subsection{Quadratic Forms}
A quadratic form in \(x\in\mathbb{R}^n\), with \(A\in\mathbb{R}^{n\times n}\), is
\begin{equation}
    q(x;A) := x^T A x .
    \label{eq:quadratic form}
\end{equation}
Every quadratic form is uniquely determined by a symmetric matrix, since
\[
q(x;A) = x^T A x = \tfrac{1}{2}\,x^T(A+A^T)x, \quad \forall x\in\mathbb{R}^n,
\]
so without loss of generality we assume \(A\) symmetric.

A symmetric matrix \(A\) can be diagonalized by an orthogonal matrix \(U\), i.e.\
\begin{equation}
    A = U^T D U,
    \label{eq:diagonalization}
\end{equation}
where \(D\) is diagonal with the eigenvalues of \(A\). Substituting into \eqref{eq:quadratic form} gives
\begin{equation}
    q(x;A) = x^T U^T D U x = (Ux)^T D (Ux) = y^T D y,
    \label{eq:transformDiag}
\end{equation}
with \(y=Ux\). Hence, learning \(q(x;A)\) reduces to learning an orthogonal matrix \(U\) and a diagonal matrix \(D\).



\subsection{Orthogonal Group: $O(p,q)$}
Consider the set defined as,
\begin{equation}
    O(p,q) := \{\, U \in GL(n,\mathbb{R}) : U^T A U = A \,\},
    \label{eq:Orthogonal group}
\end{equation}
which is a Lie subgroup of \(GL(n,\mathbb{R})\) that preserves the quadratic form \(q(x;A)=x^T A x\), i.e.\ \(q(Rx;A)=q(x;A)\) for all \(R\in O(p,q)\).  
Examples include the standard orthogonal group preserving the Euclidean norm \(x^T x\), and the Lorentz group preserving the Minkowski norm \(x^T \eta x\) with \(\eta=\mathrm{diag}(1,-1,-1,-1)\).

In this work, we focus on learning a function \( f : X \rightarrow \mathbb{R}^m \), where \( f \) is equivariant under an orthogonal group \( G \) that preserves a given or unknown quadratic form \( q(x; A) \). In other words, \( A \) may be known or unknown.

We now state our first result in the following theorem, which characterizes the set of orbits under the action of an orthogonal group \( G \). Specifically, each \( G \)-orbit corresponds to a unique value of the quadratic form.

\begin{restatable}{theorem}{orbitSpace}
\label{thm:orbit-space}
Let \( A \) be a non-zero \( n \times n \) matrix, and define \( U := \fbracket{ x \in \mathbb{R}^n : x^T A x = 0 } \). Let \( G \) be an orthogonal group that preserves the quadratic form \( x^T A x \), and consider the action of \( G \) on \( \mathbb{R}^n \setminus U \). For a given \( c \in \euclidean{} \setminus \{0\} \), define the set \( \mathcal{O}_c := \fbracket{ x \in \euclidean{n} : x^T A x = c } \). Then, the set of orbits under the action of \( G \) is given by
\begin{equation}
    \mathcal{O}\paren{G} = \fbracket{\mathcal{O}_c: c \in \euclidean{n} \setminus \{0\} \And \mathcal{O}_c \neq \varnothing}.
    \label{eq:orbit-space}
\end{equation}
\end{restatable}

Using the above result, we show that any \( G \)-equivariant function can be expressed in a canonical form that combines a norm-invariant (equivalently, \( G \)-invariant) function and a scale-invariant function. This formulation is formalized in the following theorem.

\begin{restatable}{theorem}{mainResult}
\label{thm:main-result}
Let \( A \) be a non-zero \( n \times n \) matrix, and define \( U := \fbracket{x \in \euclidean{n}: x^TAx = 0} \). Any \( G \)-equivariant function \( f: \euclidean{n} \setminus U \rightarrow \euclidean{m} \), where \( G \) is an orthogonal group that preserves the quadratic form \( x^T A x \), can be expressed as,
\begin{equation}
    f(x) = \phi_s\paren{\frac{x}{\Vert x \Vert_A}} \cdot \phi_n\paren{\Vert x \Vert_A},
    \label{eq:main-result}
\end{equation}
for some functions $\phi_s: \euclidean{n} \rightarrow G$ and $\phi_n: \euclidean{} \rightarrow \euclidean{m}$. Here, \( \Vert z \Vert_A \) denotes the norm associated with the quadratic form \( A \), defined by \( \Vert z \Vert_A = \operatorname{sign}(z^T A z) \sqrt{\vert z^T A z \vert} \).
\end{restatable}
\begin{remark}
    We would like to emphasize that $\Vert z \Vert_A$ is a pseudonorm, as it can take on negative values.
\end{remark}
The scale invariance part of the function decomposition of \eqref{eq:main-result} is illustrated in Figure~\ref{fig:scale-invariance-illustration}. 
\begin{remark}
The scale invariance function \(\phi_s\) in \eqref{eq:main-result} is required to be \(G\)-equivariant. However, because its input is normalized with respect to the quadratic form defined by \(A\), \(\phi_s\) effectively operates on a single orbit, unlike the original function \(f\) in \eqref{eq:main-result} which acts on multiple orbits. Consequently, one could incorporate an additional inductive bias toward \(G\)-equivariance for \(\phi_s\) (for example, by introducing a regularizer) beyond the inherent scale invariance. In our experiments, however, we observed that the built-in scale invariance suffices.
\end{remark}

\begin{figure}[!ht]
\centering
\resizebox{0.75\columnwidth}{!}{%
\begin{tikzpicture}[scale=2.0, font=\footnotesize]
\def\ang{70}
\fill[teal, opacity=0.5] (0,0) circle (2);
\fill[white] (0,0) circle (1.5);
\draw[thin, color=teal] (0,0) circle (2);

\fill[orange, opacity=0.5] (0,0) circle (1.5);
\fill[white] (0,0) circle (1);
\draw[thin, color=orange] (0,0) circle (1.5);

\fill[blue, opacity=0.5] (0,0) circle (1);
\draw[thin, color=blue] (0,0) circle (1);

\draw[->, very thick, color=black] (0,0) -- (2.3,0) node[right, black] {$e_1$};
\draw[->, very thick, color=black] (0,0) -- (\ang:2.3) node[right, red!80!black] {$R(\theta)\,e_1$};
\coordinate (A1) at (1,0);
\coordinate (A2) at (1.5,0);
\coordinate (A3) at (2,0);
\coordinate (B1) at (\ang:1);
\coordinate (B2) at (\ang:1.5);
\coordinate (B3) at (\ang:2);
\fill[blue!50!blue] (A1) circle (0.03) node[below right, black] {$r_1$};
\fill[orange!50!orange] (A2) circle (0.03) node[below right, black] {$r_2$};
\fill[teal!50!teal] (A3) circle (0.03) node[below right, white] {$r_3$};
\fill[blue!50!blue] (B1) circle (0.03) node[above right = 0.5pt, black] {$x$};
\fill[orange!50!orange] (B2) circle (0.03) node[above right = 0.5pt, black] {$y$};
\fill[teal!50!teal] (B3) circle (0.03) node[above right = 0.5pt, white] {$z$};
\node at (0.75,0.15) [black] {$\phi_s = I$};
\node at (\ang:1.15) [black, above left] {$\phi_s = R(\theta)$};
\end{tikzpicture}
}
\caption{For an $O(2)$-equivariant function $f$, we have $f(x) = f\paren{R(\theta)r_1 e_1} = \protect\underbrace{R(\theta)}{}f\paren{r_1 e_1}$. Similarly, $f(y) = \protect\underbrace{R(\theta)}{}f\paren{r_2 e_1}; \: f(z) = \protect\underbrace{R(\theta)}{}f\paren{r_3 e_1}.$ Hence, $\phi_s(x) = \phi_s(y) = \phi_s(z) = R(\theta)$. This property holds for general orthogonal groups as well.}
\label{fig:scale-invariance-illustration}
\end{figure}
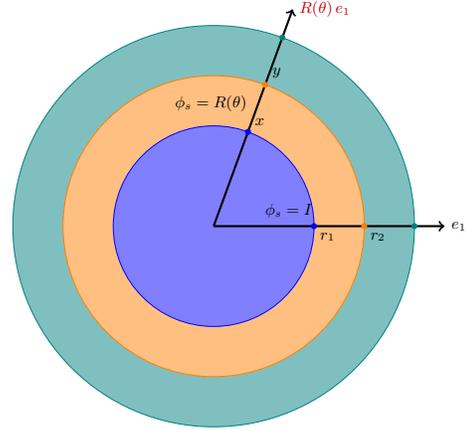

\begin{remark}
In Theorem \ref{thm:main-result}, the set \( U \) has measure zero in \( \euclidean{n} \). Therefore, including this set \( U \) in our experiments has a negligible impact on the results, as it represents an insignificant portion of the space.
\end{remark}

Next, we analyze the continuity and differentiability of the \( G \)-equivariant function as follows:

\begin{restatable}{proposition}{regularity}
\label{prop:regularity}
Consider Theorem \ref{thm:main-result}. Suppose \( A \) is either a positive or a negative definite matrix. Then, \( \phi_n \) and \( \phi_s \) are smooth (\( C^\infty \)) whenever \( f \) is smooth (\( C^\infty \)). If \( A \) is neither a positive nor a negative definite matrix, then \( \phi_n \) and \( \phi_s \) are continuous (\( C^0 \)) whenever \( f \) is continuous (\( C^0 \)).
\end{restatable}

\subsection{Diagonal group action}
We now generalize our approach to the setting where the function acts on tuples of vectors rather than on individual vectors. Specifically, let \(X\) be a \(G\)-invariant subset of \(\mathbb{R}^n \setminus U\), where
\[
U := \{ x \in \mathbb{R}^n : x^T A x = 0 \},
\]
and define the product space
\[
X_p := \underbrace{X \times X \times \cdots \times X}_{p\text{ times}}.
\]
On \(X_p\), the group \(G\) acts diagonally, that is,
\[
g \cdot (x_1, x_2, \ldots, x_p) = (g \cdot x_1,\, g \cdot x_2,\, \ldots,\, g \cdot x_p).
\]
We recall the standard result (see Appendix, Classification of $p$-tuples up to the $A$-orthogonal group) which states that any two \(p\)-tuples of vectors in \(\mathbb{R}^n\) that share the same Gram matrix (with respect to a non-degenerate symmetric bilinear form) are related by an isometry \citep{Jacobson-BasicAlgebra1}. Building upon this well-established fact, we derive a new decomposition theorem for functions equivariant under the diagonal \(G\)-action on tuples. This new result, formalized in the following theorem, extends our previous method to a broader and more general setting.

\begin{restatable}[Extended Decomposition for \(G\)-Equivariant Functions on Tuples]{theorem}{ExtendedDecomp}
\label{thm:main-result-extended}
Let \( A \) be a non-singular \( n \times n \) real matrix and define 
Let $G$ be orthogonal group as defined in eq.~\eqref{eq:Orthogonal group}. Let \( X \subset \mathbb{R}^n \) be a \( G \)-invariant set. For any positive integer \( p \), consider the diagonal action of \( G \) on the product space
\(
X_p := \underbrace{X \times X \times \cdots \times X}_{p\text{ times}},
\)
defined by, $\forall g \in G$
\[
g \cdot (x_1,  \dots, x_p) = (g \cdot x_1,\,  \dots,\, g \cdot x_p).
\]

Then every \( G \)-equivariant function  $f : X_p \to \mathbb{R}^m $ can be decomposed as
\begin{equation}
    f(x) = \phi_s \paren{\frac{x_1}{\Vert x_1 \Vert_A}} \cdot \phi_n \paren{\Gram(x)},
    \label{eq:extended-main-result}
\end{equation}
where  $x = \paren{x_1,\ldots,x_p}$ and Gram matrix
    $$\Gram(x_1,\ldots,x_p) = \begin{bmatrix}x_i^T A x_j \end{bmatrix}_{i, j \in [p]}$$
    encodes all pairwise inner products with respect to the quadratic form defined by \( A \).
\end{restatable}
Please refer to the appendix for the proofs of all the Theorems and propositions.
\begin{remark}
When \(A\) is non-singular, Theorem~\ref{thm:main-result} is a special case of Theorem~\ref{thm:main-result-extended}. Notably, Theorem~\ref{thm:main-result} remains applicable even when \(A\) is singular. In contrast, Theorem~\ref{thm:main-result-extended} would require further modifications to accommodate the singular case, and we leave these extensions for future work.
\end{remark}

\subsection{Learning Equivariant Functions}
Given data \(\{\left(x_i, y_i = f(x_i)\right)\}_{i \in [m]}\), where \(f\) is a \(G\)-equivariant function for some known or unknown orthogonal group \(G\) that preserves the quadratic form \(x^T A_0 x\), we implement the functions \(\phi_s\) and \(\phi_n\) through neural networks. We denote these networks by the same notation, with \( \phi_s \) and \( \phi_n \) representing the scale-invariant and norm-invariant components of the function, respectively.

\subsubsection{$G$ is Known}
When the group \( G \) is known, we fix the corresponding quadratic form in our network architecture, specifically fixing the matrix \( A \). In this case, we aim to minimize the following loss function:
 \begin{equation}
     \arg \min \limits_{\theta} \sum \limits_i \mathcal{L} \paren{y_i, \phi_s\paren{\frac{x}{\Vert x \Vert_A}} \cdot \phi_n\paren{\Vert x \Vert_A}}
 \label{eq:obj-sym-known}    
 \end{equation}
where \(\theta = \left\{\theta_s, \theta_n\right\}\) are the weights of the networks corresponding to \(\phi_s\) and \(\phi_n\), respectively.

\subsubsection{$G$ is Unknown: Symmetry Discovery Framework}
In this case, the goal is to discover the underlying symmetry group \( G \) by learning the appropriate matrix \( A \). We seek to minimize the following loss function:
 \begin{equation}
     \arg \min \limits_{A, \theta} \sum \limits_i \mathcal{L} \paren{y_i, \phi_s\paren{\frac{x}{\Vert x \Vert_A}} \cdot \phi_n\paren{\Vert x \Vert_A}},
 \label{eq:obj-sym-unknown}    
 \end{equation}
Thus, discovering the underlying group \( G \) corresponds to learning the appropriate symmetric matrix \( A \). However, as discussed earlier, this problem reduces to learning the appropriate orthonormal matrix \( U \) and diagonal matrix \( D \) due to the diagonalization of the symmetric matrix \( A \). Therefore, we can rewrite the objective function in eq.~\eqref{eq:obj-sym-unknown} as:
 \begin{equation}
      \arg \min \limits_{U, D, \theta} \sum \limits_i \mathcal{L} \paren{y_i, \phi_s\paren{\frac{\hat{x}}{\Vert \hat{x} \Vert_D}} \cdot \phi_n\paren{\Vert \hat{x} \Vert_D}},
     \label{eq:obj-diag}
 \end{equation}
where $\hat{x} = Ux$,  \( U \in O(n) \) and \( D \) is a diagonal matrix. In this formulation, \( U \) captures the orthogonal transformation and \( D \) scales the components of the input \( x_i \), thereby enabling the discovery of the underlying symmetry group. The block diagram of our proposed method is depcited in Figure~\ref{fig:block_diagram_1}.


For invariance tasks, we use only the \(\phi_n\) network and omit \(\phi_s\). When the group acts on the input domain via a product action, we adopt the function decomposition from \eqref{eq:extended-main-result}. This decomposition is applied both when the group is explicitly known, where we incorporate it into \eqref{eq:obj-sym-known}, and when the group is unknown, in which case it is used in \eqref{eq:obj-sym-unknown}.

\section{Discussion}
\subsection{Existing Canonical Forms}
We would like to highlight that, for $O(n)$, there exist various standard forms for equivariant functions, depending on the output tensor types and the group action. These forms can be adapted for generic orthogonal groups. However, they change depending on the tensor type and the group action, unlike our method, which provides a common form irrespective of the tensor type and the group action.

\subsection{Uniqueness and Householder Reflection}

The proposed canonical decomposition of a $G$-equivariant function is unique up to a group transformation. Furthermore, under such a transformation, the output of the scale-invariant component of the function is equivalent to a Householder reflection transformation, $R$. This can be expressed as:
\begin{equation}
    \Phi_s\left(\frac{x}{\Vert x \Vert_A}\right) = \tau\left(I - \frac{2ww^T A}{w^T A w}\right) \tau^{-1}
    \label{eq:householder}
\end{equation}
where $\tau$ is an appropriate transformation that maps the group $G$ to the standard indefinite orthogonal group $O(p,q)$.

A key property of a Householder reflection is that it is its own inverse (i.e., $R = R^{-1}$). This insight leads to a significant simplification of our proposed form, particularly when the group action on the function's output is a conjugation:
\begin{equation}
    f(g \cdot x) = g \cdot f(x) = g f(x) g^{-1}
    \label{eq:conjugation}
\end{equation}

This simplification is formalized in the following proposition.


\begin{restatable}{proposition}{propSymmForm}
\label{prop:symm_form}
Let $f$ be a $G$-equivariant function where the group action on the output is a conjugation, as defined in eq.~\eqref{eq:conjugation}. Then, the function $f(x)$ can be expressed in the symmetric form:
\begin{equation}
    f(x) = \Phi_s\left(\frac{x}{\Vert x \Vert_A}\right) \Phi_n\left(\Vert x \Vert_A\right) \Phi_s\left(\frac{x}{\Vert x \Vert_A}\right)
    \label{eq:symmetric_form}
\end{equation}
\end{restatable}

\section{Experiments}

We evaluate \textbf{$G$-Ortho-Nets}, our proposed method for learning $G$-equivariant functions through decomposition into norm-invariant and scale-invariant components as given in eq. \eqref{eq:main-result}. The quadratic form $A$ is learned jointly with the decomposed functions, enabling simultaneous discovery of the underlying symmetry group $O(p,q)$ which is defined in eq. \eqref{eq:Orthogonal group}.

\subsection{Baseline Methods}
\begin{table*}[!htbp]
\centering
\scriptsize
\caption{Suitability of related symmetry discovery methods for our setting. Only Augerino and LieGG provide both a usable prediction model and applicability to continuous equivariance tasks, motivating our choice of baselines. Our method further distinguishes itself by offering an \textbf{interpretable} model that trains in a \textbf{single stage}. \checkmark = supported, \texttimes = not supported/limited.}
\label{tab:baseline_suitability}
\renewcommand{\arraystretch}{1.4}
\setlength{\tabcolsep}{8pt}
\begin{tabular}{@{}lcccl@{}}
\toprule
\textbf{Method} & 
\textbf{Prediction} & 
\textbf{Continuous} & 
\textbf{Interpretable} & 
\textbf{Limitations/Notes} \\
& \textbf{ Model} & \textbf{Equivariance} & \textbf{Prediction Model} & \\
\midrule
\cellcolor{blue!8}\textbf{G-Ortho-Nets (Ours)} & 
\cellcolor{blue!8}\checkmark & 
\cellcolor{blue!8}\checkmark & 
\cellcolor{blue!8}\checkmark & 
\cellcolor{blue!8}\textit{Interpretable prediction model, single-stage} \\
\midrule[\heavyrulewidth]
\rowcolor{gray!20}
Augerino\textsuperscript{\textdagger} & \checkmark & \checkmark & \texttimes & Sample averaging required at inference \\
\rowcolor{gray!20}
LieGG\textsuperscript{\textdagger} & \checkmark & \checkmark & \texttimes & Post-training analysis needed \\
\midrule
LieGAN & \texttimes & \checkmark & \texttimes & No prediction model \\
LIG & \checkmark & \texttimes & \texttimes & Requires synchronized output flows, two stage \\
SDBA & \checkmark & \texttimes & \texttimes & Limited to single generator, two stage \\
BNN & \checkmark & \texttimes & \checkmark & Appropriate for discrete groups \\
\bottomrule
\multicolumn{5}{@{}l@{}}{\textsuperscript{\textdagger}Selected as baselines (only methods suitable for continuous equivariance with prediction models)}
\end{tabular}
\end{table*}

We compare $G$-Ortho-Nets against two recent approaches for symmetry discovery:

\paragraph{LieGG} (\cite{moskalev2022liegg}) retrieves infinitesimal generators of symmetries by solving a matrix nullspace equation derived from Lie group theory. Using singular value decomposition of the polarization matrix, it extracts the Lie algebra basis and quantifies \emph{symmetry bias} and \emph{symmetry variance}, measuring how closely learned generators approximate true underlying symmetries.

\paragraph{Augerino} (\cite{benton2020learning}) learns invariances and equivariances by parameterizing a distribution over data augmentations. The model jointly optimizes network and augmentation parameters to discover symmetries directly from training data. Unlike LieGG's post-hoc extraction, Augerino induces soft equivariance during training.

The rationale for selecting these two methods as baselines is summarized in Table~\ref{tab:baseline_suitability}.

\subsection{Experimental Tasks}

We evaluate on three tasks spanning synthetic regression, physical simulation, and high-energy physics applications. Table~\ref{tab:tasks} summarizes the tasks and their corresponding symmetry groups.

\begin{table}[h]
\centering
\caption{Experimental tasks and symmetry groups.}
\label{tab:tasks}
\scriptsize
\begin{tabular}{@{}lc@{}}
\toprule
\textbf{Task} & \textbf{Group} \\
\midrule
Synthetic regression & $O(2,2)$ \\
Moment of inertia & $O(3)$ \\
Top quark tagging & $O(1,3)$ \\
\bottomrule
\end{tabular}
\end{table}

\subsubsection{Task 1: Synthetic Regression}

We construct a target function exhibiting equivariance under the indefinite orthogonal group $O(2,2)$:
\begin{equation}
    f(x) = 9\left(x x^\top A x x^\top A\right) + 2\left(x x^\top A\right),
\end{equation}
where $A \in \mathbb{R}^{4 \times 4}$ is symmetric and indefinite with signature $(2,2)$—two positive and two negative eigenvalues. The function satisfies $f(g \cdot x) = g \, f(x) \, g^\top$ for all $g \in G$.
This controlled setting allows precise evaluation of symmetry recovery.

\subsubsection{Task 2: Moment of Inertia Prediction}

Given $n$ point masses $m_i$ at positions $x_i \in \mathbb{R}^3$, we predict the inertia tensor:
\begin{equation}
    f(x,m) = \sum_{i=1}^n m_i \left( \|x_i\|^2 I - x_i x_i^\top \right).
\end{equation}
This physically meaningful function is equivariant under the group $O(3)$:
\begin{equation}
    f(g \cdot x, m) = g \, f(x,m) \, g^\top, \quad \forall g \in O(3).
\end{equation}
This task tests the ability to recover well-known physical symmetries from data.

\subsubsection{Task 3: Top Quark Tagging}

We classify jets as originating from top quarks versus light quarks using jet constituent four-vectors. The classification label is invariant under Lorentz transformations corresponding to the group
\begin{equation}
    G = O(1,3), \quad A = \eta = \mathrm{diag}(1, -1, -1, -1).
\end{equation}
This high-dimensional real-world task evaluates Lorentz invariance discovery in particle physics applications.

\subsection{Evaluation Metrics}

We assess performance using three complementary metrics:

\paragraph{Prediction Accuracy.} We measure the mean squared error (MSE) for regression tasks (Tasks 1--2) and the classification accuracy for the tagging task (Task 3). This quantifies how well models approximate the target function.

\paragraph{Quadratic Form Recovery.} We compute the cosine similarity between the learned and true quadratic forms:
\begin{equation}
    \cos(A_0, A_{\text{learnt}}) = \frac{\langle A_{\text{learnt}}, A_0 \rangle_F}{\|A_{\text{learnt}}\|_F \, \|A_0\|_F},
\end{equation}
where $\langle \cdot, \cdot \rangle_F$ denotes the Frobenius inner product. For $G$-Ortho-Nets, $A$ is learned directly; for LieGG and Augerino, we recover $A$ from the learned generators $X$ via the constraint $X^\top A + A X = 0$.

\paragraph{Lie Algebra Comparison.} We compare the learned Lie algebra $\mathfrak{g}_{\text{learnt}}$ with the true Lie algebra $\mathfrak{g}_0$ using the principal angles. Given orthonormal bases for both subspaces, the principal angles $\theta_1 \leq \theta_2 \leq \cdots \leq \theta_k \in [0, \pi/2]$ characterize the geometric relationship between them. We compute the projection distance as:
\begin{equation}
    d_{\text{PA}}(\mathfrak{g}_0, \mathfrak{g}_{\text{learnt}}) = \sqrt{\sum_{i=1}^{k} \sin^2(\theta_i)},
\end{equation}
which measures the overall misalignment between the two Lie algebra subspaces. A value of 0 indicates perfect alignment, while larger values indicate greater deviation.

\subsection{Results}


\subsubsection{Quantitative results}
\Cref{tab:regression_tasks,tab:classification_task} summarize performance across all tasks. 
We evaluate prediction quality (MSE or accuracy), recovery of the quadratic form via 
$\cos(A_0,A_{\text{learnt}})$, and Lie-algebra alignment using projection distance 
$d_{\text{PA}}(\mathfrak{g}_0,\mathfrak{g}_{\text{learnt}})$ (principal angles). 
While LieGG achieves competitive performance in certain cases, it does not yield an interpretable prediction model and requires additional post-hoc analysis to extract Lie algebra generators.

\subsubsection{Qualitative results}
\Cref{fig:results} compares ground-truth and learned quadratic forms for the top-tagging and synthetic tasks, 
demonstrating faithful recovery of the target structure. $G$-Ortho-Nets achieve strong performance across both metrics and tasks. 
By explicitly parameterizing the quadratic form $A$, our method enables direct symmetry discovery, rather than inferring it indirectly from generators. 
On the top-tagging benchmark in particular, it attains state-of-the-art accuracy while most closely recovering the Lorentz structure, highlighting that encoding physically meaningful symmetries can improve generalization in high-energy physics.


\begin{table*}[t]
\centering
\caption{Regression tasks. Best results in \textbf{bold}. 
$\cos(A_0, A_{\text{learnt}})$ is cosine similarity between true and learned $A$; 
$d_{\text{PA}}(\mathfrak{g}_0, \mathfrak{g}_{\text{learnt}})$ is projection distance from principal angles.}
\label{tab:regression_tasks}
\begin{tabular}{@{}lcccc@{}}
\toprule
\textbf{Task} & \textbf{Method} & \textbf{Pred. Error} & \textbf{$\cos(A_0, A_{\text{learnt}})$} & \textbf{$d_{\text{PA}}(\mathfrak{g}_0, \mathfrak{g}_{\text{learnt}})$} \\
\midrule
\multirow{3}{*}{Synthetic} 
    & $G$-Ortho-Nets & $1.80 \times 10^{-3}$ & \textbf{0.99} & \textbf{0.24} \\
    & LieGG          & $\mathbf{7.53 \times 10^{-4}}$ & 0.93 & 1.80 \\
    & Augerino       & $2.63 \times 10^{0}$ & 0.30 & 1.99 \\
\midrule
\multirow{3}{*}{Inertia} 
    & $G$-Ortho-Nets & $\mathbf{1.88 \times 10^{-3}}$ & \textbf{0.99} & 0.15 \\
    & LieGG          & $2.61 \times 10^{-3}$ & \textbf{0.99} & \textbf{0.06} \\
    & Augerino       & $7.86 \times 10^{-3}$ & 0.57 & 1.26 \\
\bottomrule
\end{tabular}
\end{table*}

\begin{table*}[t]
\centering
\caption{Top-tagging classification. Best results in \textbf{bold}.}
\label{tab:classification_task}
\begin{tabular}{@{}lcccc@{}}
\toprule
\textbf{Task} & \textbf{Method} & \textbf{Accuracy} & \textbf{$\cos(A_0, A_{\text{learnt}})$} & \textbf{$d_{\text{PA}}(\mathfrak{g}_0, \mathfrak{g}_{\text{learnt}})$} \\
\midrule
\multirow{3}{*}{Top tagging} 
    & $G$-Ortho-Nets & \textbf{90.44} & \textbf{0.9999} & \textbf{0.0214} \\
    & LieGG          & 83.10 & 0.9991 & 1.0012 \\
    & Augerino       & 51.80 & 0.5000 & 2.1830 \\
\bottomrule
\end{tabular}
\end{table*}


\begin{figure}[t]
\centering
\begin{subfigure}[b]{0.47\columnwidth}
    \centering\includegraphics[width=\textwidth]{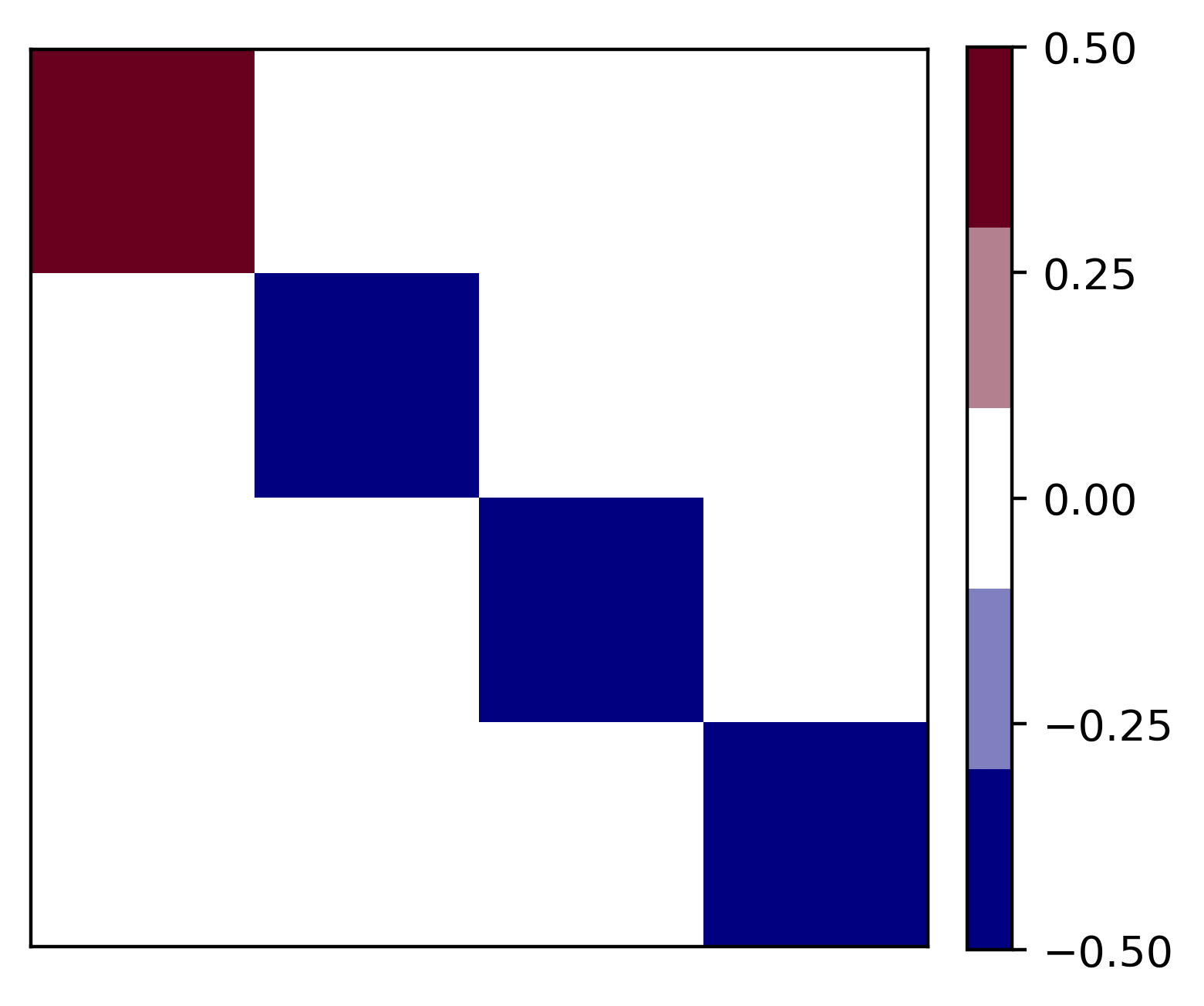}
    \caption{Ground-truth $A$}
\end{subfigure}
\hfill
\begin{subfigure}[b]{0.47\columnwidth}
    \centering\includegraphics[width=\textwidth]{Images/iclr_22_1.png}
    \caption{Learned $A$}
\end{subfigure}

\vskip\baselineskip

\begin{subfigure}[b]{0.47\columnwidth}
    \centering\includegraphics[width=\textwidth]{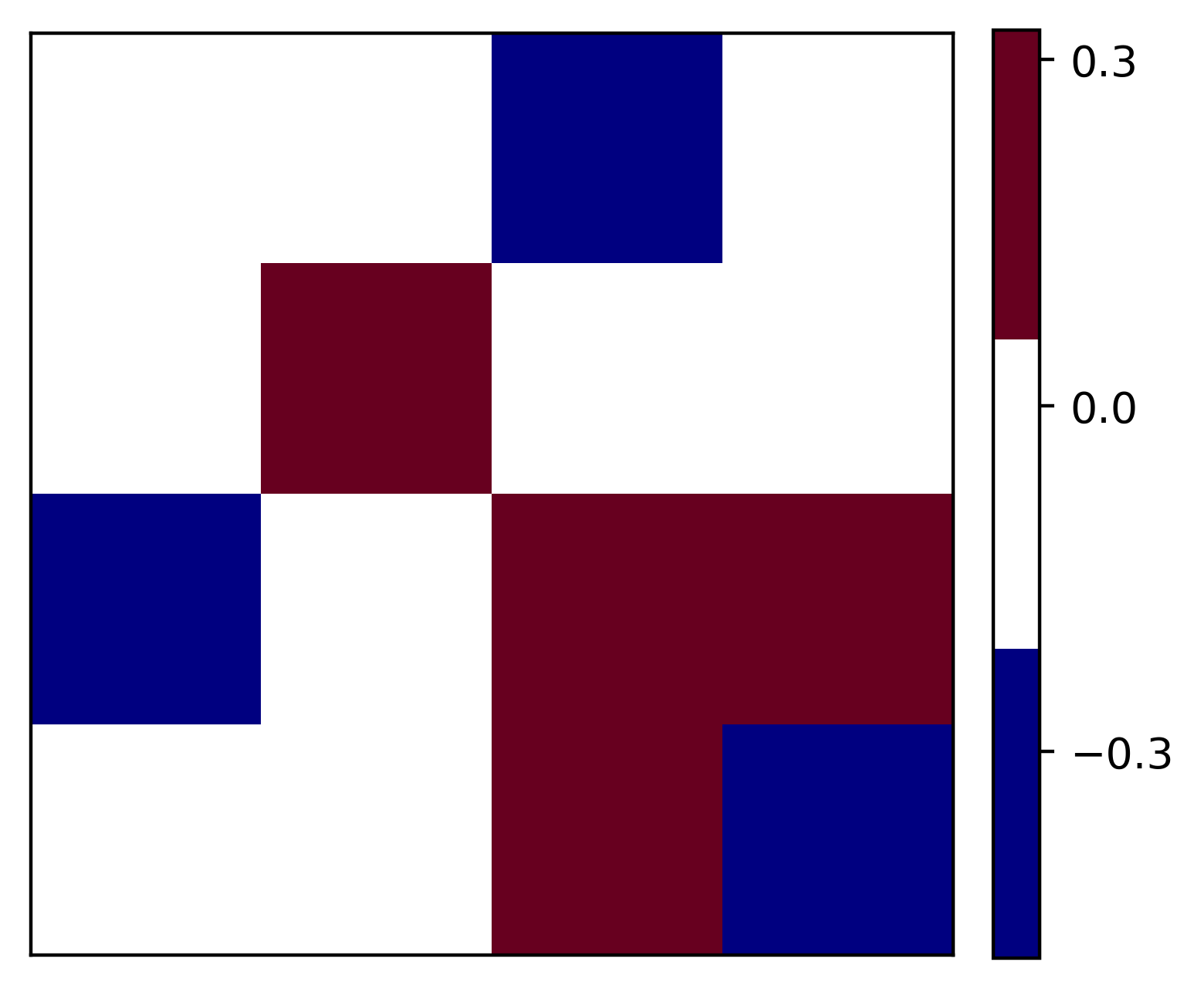}
    \caption{Ground-truth $A$}
\end{subfigure}
\hfill
\begin{subfigure}[b]{0.47\columnwidth}
    \centering\includegraphics[width=\textwidth]{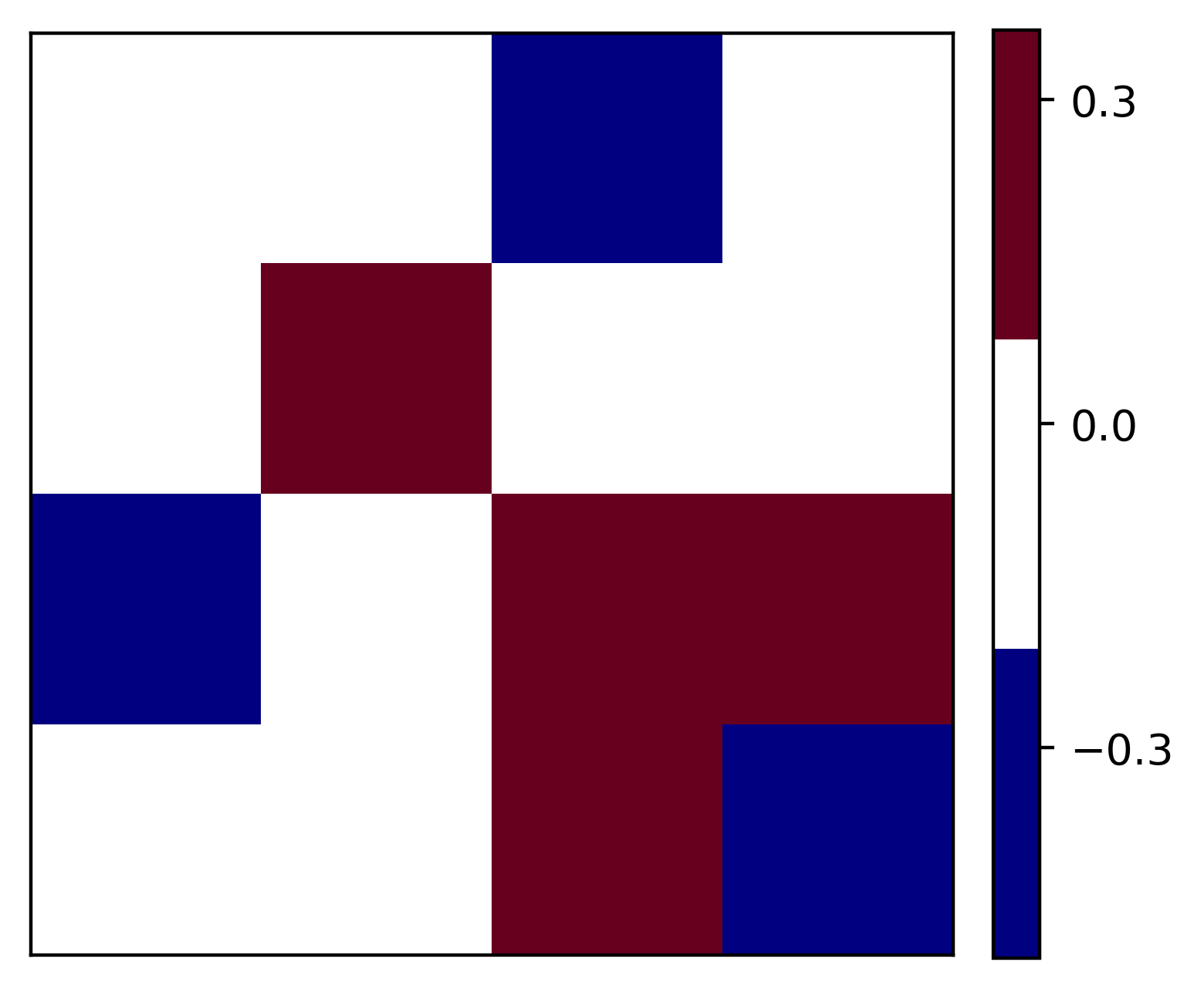}
    \caption{Learned $A$}
\end{subfigure}
\caption{Comparison of ground-truth and learned $A$ matrices for top-tagging (top) and synthetic regression (bottom) using $G$-Ortho-Nets.}
\label{fig:results}
\end{figure}

Please refer to the appendix for additional experiments and analysis.


\subsection{Limitations}
Our framework is currently limited to orthogonal groups, as these are exactly the transformations that preserve quadratic forms. While this restriction excludes more general classes such as affine or diffeomorphic symmetries, it provides a clear advantage: interpretability. By explicitly learning the quadratic form $A$, our model yields a physically meaningful and directly interpretable predictor, rather than only an implicit or post-hoc symmetry characterization. This trade-off mirrors the broader tension in symmetry learning between generality and interpretability. We view our approach as complementary to methods such as LieGAN, Augerino or LIG, which target broader classes but do not produce usable equivariant predictors.

\section{Conclusion}
We presented a quadratic-form framework for learning group-equivariant functions, focusing on orthogonal groups. We proved a canonical decomposition into \emph{norm-invariant} and \emph{scale-invariant} components, and extended it to tuples under diagonal actions, where the equivariant map splits into an angular term (from the normalized first vector) and a scale-invariant term (from the full Gram matrix), thereby capturing interdependencies while preserving symmetry. Experiments across synthetic and physics-inspired tasks showed reliable symmetry discovery and effective equivariant learning. Future work includes extending beyond orthogonal groups and accommodating richer group actions.

\bibliographystyle{plainnat}
\bibliography{example_paper}
\section*{Checklist}

\begin{enumerate}

  \item For all models and algorithms presented, check if you include:
  \begin{enumerate}
    \item A clear description of the mathematical setting, assumptions, algorithm, and/or model. [Yes] We provide a clear mathematical formulation of the problem and describe the proposed model and assumptions
    \item An analysis of the properties and complexity (time, space, sample size) of any algorithm. [No] We do not provide a formal complexity analysis, but we discuss empirical efficiency.
    \item (Optional) Anonymized source code, with specification of all dependencies, including external libraries. [No] Source code is not included in this submission, but we plan to release it in the camera-ready version.
  \end{enumerate}

  \item For any theoretical claim, check if you include:
  \begin{enumerate}
    \item Statements of the full set of assumptions of all theoretical results. [Yes] All assumptions underlying our theoretical results are explicitly stated
    \item Complete proofs of all theoretical results. [Yes] Complete proofs are provided in the Appendix.
    \item Clear explanations of any assumptions. [Yes] We provide clear explanations of all assumptions alongside the results.     
  \end{enumerate}

  \item For all figures and tables that present empirical results, check if you include:
  \begin{enumerate}
    \item The code, data, and instructions needed to reproduce the main experimental results (either in the supplemental material or as a URL). [No] We do include code, data to reproduce the main experimental results.
    \item All the training details (e.g., data splits, hyperparameters, how they were chosen). [Yes] Training details such as data splits and hyperparameters are described in the supplementary section.
    \item A clear definition of the specific measure or statistics and error bars (e.g., with respect to the random seed after running experiments multiple times). [No] We do not report evaluation metrics with error bars across multiple random seeds.  
    \item A description of the computing infrastructure used. (e.g., type of GPUs, internal cluster, or cloud provider). [Yes] The computing infrastructure, including GPU specifications are described in the supplementary section.
  \end{enumerate}

  \item If you are using existing assets (e.g., code, data, models) or curating/releasing new assets, check if you include:
  \begin{enumerate}
    \item Citations of the creator If your work uses existing assets. [Yes] We cite all assets used in our work with proper references. 
    \item The license information of the assets, if applicable. [Not Applicable] License information of the used assets is included where applicable.
    \item New assets either in the supplemental material or as a URL, if applicable. [Not Applicable] We do not release new assets in this work. 
    \item Information about consent from data providers/curators. [Not Applicable] No data requiring consent was collected. 
    \item Discussion of sensible content if applicable, e.g., personally identifiable information or offensive content. [Not Applicable] The work does not involve sensitive or personally identifiable content.  
  \end{enumerate}

  \item If you used crowdsourcing or conducted research with human subjects, check if you include:
  \begin{enumerate}
    \item The full text of instructions given to participants and screenshots. [Not Applicable] No human participants were involved.
    \item Descriptions of potential participant risks, with links to Institutional Review Board (IRB) approvals if applicable. [Not Applicable] No IRB approval was required. 
    \item The estimated hourly wage paid to participants and the total amount spent on participant compensation. [Not Applicable] No participant compensation was involved.
  \end{enumerate}

\end{enumerate}

\clearpage
\appendix
\thispagestyle{empty}

\onecolumn
\aistatstitle{Supplementary Material}

\section*{Index}
\vspace{-0.25\baselineskip}
\noindent\small

\normalsize

\noindent\rule{\linewidth}{0.4pt}

{
\small
\setlength{\parskip}{0pt}
\setlength{\itemsep}{0.3ex}
\renewcommand\labelitemi{$\triangleright$}

\begin{itemize}
\item \hyperref[sec:Additional Experiments]{Additional Experiments} \dotfill \pageref{sec:Additional Experiments}
\vspace{0.25cm}
\item \hyperref[sec:Notation and Standing Assumptions]{Notation and Standing Assumptions} \dotfill \pageref{sec:Notation and Standing Assumptions}
\vspace{0.25cm}  
\item \hyperref[sec:Additional Results and Complete Proofs]{Additional Results and Complete Proofs} \dotfill \pageref{sec:Additional Results and Complete Proofs}
\vspace{0.25cm}
  \begin{itemize}
      \item \hyperref[sec:canonical-alignment-posdef]{Canonical Alignment in the Positive Definite Case} \dotfill \pageref{sec:canonical-alignment-posdef}
      \vspace{0.25cm}
      \item \hyperref[sec:canonical-alignment-possemidef]{Canonical Alignment in the Positive Semi-Definite Case} \dotfill \pageref{sec:canonical-alignment-possemidef}
      \vspace{0.25cm}
      \item \hyperref[sec:canonical-alignment-indef]{Canonical Alignment in the Indefinite Case} \dotfill \pageref{sec:canonical-alignment-indef}
      \vspace{0.25cm}
      \item \hyperref[sec:orbit-space-proof]{Orbit Space and Quadratic-Form Norm Correspondence} \dotfill \pageref{sec:orbit-space-proof}
      \vspace{0.25cm}
      \item \hyperref[sec:decomposition]{Decomposition into Scale- and Norm-Invariant Components} \dotfill \pageref{sec:decomposition}
      \vspace{0.25cm}
      \item \hyperref[sec:reg-cond]{Regularity conditions} \dotfill \pageref{sec:reg-cond}
      \vspace{0.25cm}
      \item \hyperref[sec:diag-action-tuples]{Extension to Diagonal action} \dotfill \pageref{sec:diag-action-tuples}
      \vspace{0.25cm}
      \item \hyperref[sec:Householder-reflection]{Householder reflection} \dotfill \pageref{sec:Householder-reflection}
  \end{itemize}
\end{itemize}
}

\section{Additional Experiments}
\label{sec:Additional Experiments}
\subsection{Noise robustness}
\label{subsec:Noise robustness}
To evaluate the robustness of both methods to label corruption, we systematically vary the standard deviation of Gaussian noise added to the training labels. The underlying task is a synthetic polynomial regression problem, where the ground-truth function is given by $f(x) = 9 (x x^\top A)^2 + 2 (x x^\top A)$. The goal is to assess how prediction accuracy degrades under increasing noise levels and to compare the noise tolerance of LieGG and our proposed $G$-Ortho-Nets. 
\Cref{tab:label_noise_results} reports the prediction error across different noise levels, highlighting the relative stability of each method.

\begin{table}[!h]
\centering
\caption{Comparison of $G$-Ortho-Nets and LieGG under varying label noise levels. 
Lower prediction error (MSE) and projection distance $d_{\text{PA}}$ indicate better performance, while higher cosine similarity $\cos(A_0, A_{\text{learnt}})$ indicates closer alignment with the ground-truth quadratic form.}
\label{tab:label_noise_results}
\begin{tabular}{@{}lcccc@{}}
\toprule
\textbf{Std dev} & \textbf{Method} & \textbf{Pred. Error (MSE)} & $\boldsymbol{\cos(A_0, A_{\text{learnt}})}$ & $\boldsymbol{d_{\text{PA}}(\mathfrak{g}_0, \mathfrak{g}_{\text{learnt}})}$ \\ 
\midrule
\multirow{2}{*}{1.0}   
    & $G$-Ortho-Nets & \textbf{1.0161}  & \textbf{0.9778} & \textbf{0.3504} \\
    & LieGG          & 43.7929 & 0.2900 & 1.7526 \\ 
\midrule
\multirow{2}{*}{0.5}   
    & $G$-Ortho-Nets & \textbf{0.2633}  & \textbf{0.9882} & \textbf{0.2872} \\
    & LieGG          & 10.9812 & 0.2468 & 1.7108 \\ 
\bottomrule
\end{tabular}
\end{table}

The results in \Cref{tab:label_noise_results} demonstrate the superior robustness and structural fidelity of the proposed $G$-Ortho-Nets compared to the LieGG baseline under increasing label noise. 
Even when subjected to substantial noise ($\sigma = 1$), $G$-Ortho-Nets maintain a remarkably low prediction error (MSE $\approx 1.0$) and a high cosine similarity with the ground-truth quadratic form ($\cos(A_0, A_{\text{learned}}) \approx 0.98$), indicating that the learned symmetry remains closely aligned with the true underlying structure. 
In contrast, the LieGG method exhibits severe degradation, with prediction errors exceeding two orders of magnitude and a sharp drop in alignment metrics. 
These findings highlight that explicitly parameterizing the quadratic form and enforcing equivariance through the proposed  decomposition enables stable learning of both the predictive function and the latent symmetry, even in the presence of substantial label corruption.

\section{Notation and Standing Assumptions}
\label{sec:Notation and Standing Assumptions}
\begin{itemize}
    \item \textbf{Quadratic form and group definition.}  
    Let $A \in \mathbb{R}^{n\times n}$ be a symmetric matrix.  
    Define the generalized orthogonal group
    \[
        G \;=\; \{\, Q \in \mathrm{GL}_n \;:\; Q^\top A Q = A \,\}.
    \]

    \item \textbf{Indefinite norm.}  
    For $x \in \mathbb{R}^n$, define the $A$-norm (or pseudo-norm)
    \[
        \|x\|_{A} \;=\; \operatorname{sign}(x^\top A x)\,\sqrt{|x^\top A x|}.
    \]
    This norm generalizes the Euclidean case and may take positive or negative values depending on the signature of $A$.

    \item \textbf{Null cone.}  
    We exclude the \emph{null cone}
    \[
        U \;=\; \{\, x \in \mathbb{R}^n : x^\top A x = 0 \,\},
    \]
    which has Lebesgue measure zero.  
    All maps and actions below are defined on $\mathbb{R}^n \setminus U$.

    \item \textbf{Group action on inputs.}  
    The group $G$ acts on the input space by the standard linear action
    \[
        Q \cdot x \;=\; Qx.
    \]

    \item \textbf{Group action on outputs.}  
    When the output space is matrix-valued, we consider the conjugation action
    \[
        Q \cdot Y \;=\; Q Y Q^{-1},
    \]
    which preserves the natural geometric structure of the output space.
\end{itemize}

\section{Additional Results and Complete Proofs}
\label{sec:Additional Results and Complete Proofs}
Additional theoretical results and the complete proof of Theorem. \ref{thm:orbit-space} are provided here including supporting Lemmas.

\subsection{Canonical Alignment in the Positive Definite Case}
\label{sec:canonical-alignment-posdef}
We begin with the positive definite case, where the quadratic form \(x^\top D x\) 
defines a \emph{weighted Euclidean geometry} determined by the positive diagonal 
entries of \(D\). In this setting, all directions are comparable through the group 
of transformations that preserve this weighted norm. The goal of this subsection 
is to establish a \emph{canonical alignment result} showing that any vector can be 
mapped to a fixed reference direction through an appropriate transformation in the 
orthogonal group \(G = \{A \in \mathbb{R}^{n \times n} : A^\top D A = D\}\). 
This alignment construction will serve as a fundamental building block in the proof 
of \textbf{Theorem~\ref{thm:orbit-space}}, which characterizes the orbits of this group 
and their correspondence with quadratic-form norms.
\begin{lemma}
\label{lemma:canonical}
Let \( D \in \euclidean{n \times n} \) be a diagonal matrix defined as:
\begin{equation}
    D = \sum_{i=1}^n d_i E_{i,i},
    \label{eq:diag-all-pos}
\end{equation}
where \( d_1, \dots, d_n > 0 \), and  \( E_{i,j} = e_ie_j^T \) with \( e_i \) denoting the standard basis vectors in \(\euclidean{n}\). Let \( x = \sum_{i=1}^n x_i e_i \in \euclidean{n} \). Let \( G \) be the orthogonal group preserving the quadratic form \( x^T D x \). Then, there exists \( A \in G \) such that:
\begin{equation}
    A x = \alpha e_1,
    \label{eq:lemma1-claim}
\end{equation}
where \( \alpha \) is an appropriate scaling factor such that \( x^T D x = x^T A^T D A x \).
\end{lemma}

\begin{proof}
Define \( S = \sum_{i=1}^n \sqrt{d_i} e_i e_i^T \). Then, we have:
\begin{equation}
    S^{-1} D S^{-1} = I,
    \label{eq:SiDSi}
\end{equation}
where \( I \) is the identity matrix. 

Consider the standard orthogonal group:
\begin{equation*}
    O(n) = \fbracket{R \in GL\paren{n, \euclidean{}} : R^T R = I}.
\end{equation*}
We know that if \( y^T y = z^T z \) for any \( y, z \in \euclidean{n} \), then there exists \( R \in O(n) \) such that:
\begin{equation*}
    R y = z.
\end{equation*}
Applying this result, we find \( R \in O(n) \) such that:
\begin{equation}
    R S x = \beta e_1,
\end{equation}
where \( \beta \) is a scaling factor that ensures that the Euclidean norms of \( S x \) and \( \beta e_1 \) are equal.

Now, define \( A = S^{-1} R S \). Since \( S^{-1} \) is also a diagonal matrix, we compute:
\begin{align}
    A x &= S^{-1} R S x \nonumber \\
    &= S^{-1} \beta e_1 \nonumber \\
    &= \alpha e_1,
\end{align}
where \( \alpha = \frac{\beta}{\sqrt{d_1}} \) is an appropriate scaling factor.


Next, we show that \( A \in G \). Consider \( A^T D A \). Since $S$ is a diagonal matrix, $S^T = S$ and $(S^{-1})^T = S^{-1}$.
\begin{align*}
    A^T D A &= (S^{-1} R S)^T D (S^{-1} R S) \\
            &= (S^T R^T (S^{-1})^T) D (S^{-1} R S) \\
            &= (S R^T S^{-1}) D (S^{-1} R S) && \text{Since } S^T=S \\
            &= S R^T (S^{-1} D S^{-1}) R S && \text{By associativity} \\
            &= S R^T I R S && \text{Since } S=\sqrt{D}, S^{-1}DS^{-1}=I \\
            &= S R^T R S && \text{Since } R \in O(n), R^TR=I \\
            &= S I S = S^2 = D.
\end{align*}
This shows conclusively that $A \in G$.
\end{proof}
\begin{remark}
The Lemma. \ref{lemma:canonical} is also applicable when $d_i < 0,\; \forall i \in  [n]$. We just need to appropriately modify the matrix $S$.
\end{remark}

Lemma~\ref{lemma:canonical} thus guarantees that, in the positive definite setting, 
every orbit under the \(D\)-orthogonal group admits a canonical representative obtained 
by aligning an arbitrary vector with the first coordinate axis. 
This serves as the foundational case for proving \textbf{Theorem~\ref{thm:orbit-space}}; 
the subsequent subsections extend this reasoning to \emph{semi-definite} and 
\emph{indefinite} quadratic forms, where certain directions may lie on null subspaces 
or carry opposite signs.

\subsection{Canonical Alignment in the Positive Semi-Definite Case}
\label{sec:canonical-alignment-possemidef}
\begin{lemma}
\label{lemma:psd}
Let \( D \in \euclidean{n \times n} \) be a diagonal matrix defined as:
\begin{equation}
    D = \sum\limits_{i=1}^{k-1} d_i E_{i,i}\;   \;:
    \label{eq:diag-pos-zero}
\end{equation}
where \( d_1, \dots, d_{k-1} > 0 \), with $k \neq 1$, and \( E_{i,j} = e_ie_j^T \) with \( e_i \) denoting the standard basis vectors in \(\euclidean{n}\). Consider a vector \( y \in \euclidean{n} \) of the form:
\begin{equation}
    y = a e_1 + b e_k,
    \label{eq:x-vector}
    \nonumber
\end{equation}
where \( a, b \in \euclidean{} \). Let \( q(x;D) = x^T D x \) be the quadratic form associated with \( D \), and let \( G = O(n, q) \) denote the orthogonal group that preserves this quadratic form, i.e.,
\begin{equation}
    G = \{ R \in GL(n, \mathbb{R}) : R^T D R = D \}.
    \label{eq:orthogonal-group}
    \nonumber
\end{equation}
Then, there exists a matrix \( A \in G \) such that:
\begin{equation}
    A y = a e_1.
    \label{eq:lemma-claim}
\end{equation}
\end{lemma}

\begin{proof}
We construct \( A \) explicitly and verify the required properties. Define \( R \in \euclidean{n \times n} \) as:
\begin{equation}
    R =   \sqrt{d_1}ab \, E_{k,1}  -  d_1a^2  \, E_{k,k} + \sum_{i \in [n],\; i\neq 1,\; i \neq {k}} E_{i,i}.
    \label{eq:lemma1-R}
\end{equation}
Define \( S \in \euclidean{n \times n} \) as:
\begin{equation}
    S = \sqrt{d_1}E_{1,1} +  \sum_{i \in [n],\; i\neq 1} E_{i,i}.
    \label{eq:S-candidate}
\end{equation}
Then:
\begin{equation}
    S^{-1} = \frac{1}{\sqrt{d_1}}E_{1,1} + \sum_{i \in [n],\; i\neq 1} E_{i,i}.
    \label{eq:S-candidate-inv}
    \nonumber
\end{equation}

Let $A = S^{-1}RS$. We claim that $A^TDA = D$ and $Ay = ae_1$

\paragraph{Step 1: Show \( A^T D A = D \).}
Using \( A = S^{-1} R S \), we compute:
\begin{align}
    A^TDA &= SR^TS^{-1}DS^{-1}RS 
    \nonumber
\end{align}
First, compute \( S^{-1} D S^{-1} \):
\begin{align}
    S^{-1}DS^{-1} &= \paren{\frac{1}{\sqrt{d_1}}E_{1,1} + \sum_{i \in [n],\; i\neq 1} E_{i,i}} \paren{\sum\limits_{i=1}^{k-1} d_i E_{i,i}}    \paren{\frac{1}{\sqrt{d_1}}E_{1,1} + \sum_{i \in [n],\; i\neq 1} E_{i,i}} \nonumber \\
    &= E_{1,1} + \sum \limits_{i=2}^{k-1} d_iE_{i,i}.
\end{align}

Consider \(RS = \left(SR^T\right)^T\),
\begin{align}
    RS &= \paren{ \sqrt{d_1}ab \, E_{k,1}  -  d_1a^2  \, E_{k,k} + \sum_{i \in [n],\;  i \neq k} E_{i,i}} \paren{\sqrt{d_1}E_{1,1} +  \sum_{i \in [n],\; i\neq 1} E_{i,i}} \nonumber \\
    &= d_1abE_{k,1} + \sqrt{d_1}E_{1,1} - d_1a^2E_{k,k} + \sum_{i \in [n],\; i\neq 1,\;i\neq k} E_{i,i}.
\end{align}

Thus, we have,
\begin{align}
    A^TDA &= \paren{SR^T} \paren{S^{-1}DS^{-1}} \paren{RS} \nonumber\\
    & = \paren{d_1 abE_{1,k} + \sqrt{d_1}E_{1,1} - d_1 a^2E_{k,k} + \sum_{i \in [n],\; i\neq 1,\;i\neq k} E_{i,i}} \paren{E_{1,1} + \sum \limits_{i=2}^{k-1} d_iE_{i,i}} \paren{RS} \nonumber\\
    &=\paren{\sqrt{d_1}E_{1,1} + \sum \limits_{i=2}^{k-1} d_iE_{i,i}} \paren{RS} \nonumber\\
    &=\paren{\sqrt{d_1}E_{1,1} + \sum \limits_{i=2}^{k-1} d_iE_{i,i}} \paren{d_1 abE_{k,1} + \sqrt{d_1}E_{1,1} - d_1 a^2E_{k,k} + \sum_{i \in [n],\; i\neq 1,\;i\neq k} E_{i,i}} \nonumber\\
    &= d_1E_{1,1} + \sum \limits_{i=2}^{k-1} d_iE_{i,i} \nonumber\\
    & = \sum \limits_{i=1}^{k-1} d_iE_{i,i} \nonumber \\
    &=D
    \label{eq:lemma1-AtDA}
\end{align}

Consider $A = S^{-1}RS$
\begin{align}
A &= S^{-1}\paren{RS} \nonumber \\ 
&=S^{-1}\paren{d_1abE_{k,1} + \sqrt{d_1}E_{1,1} - d_1a^2E_{k,k} + \sum_{i \in [n],\; i\neq 1,\;i\neq k} E_{i,i}}\nonumber\\
&=\paren{\frac{1}{\sqrt{d_1}}E_{1,1} + \sum_{i \in [n],\; i\neq 1} E_{i,i}}\paren{d_1abE_{k,1} + \sqrt{d_1}E_{1,1} - d_1a^2E_{k,k} + \sum_{i \in [n],\; i\neq 1,\;i\neq k} E_{i,i}} \nonumber\\
&=E_{1,1} + d_1abE_{k,1} - d_1 a^2E_{k,k} + \sum_{i \in [n],\; i\neq 1,\;i\neq k} E_{i,i}.
\label{eq:lemma1-A}
\end{align}

\textbf{Step 2}: Verify \( A y = a e_1 \)

We are given the vector $y = ae_1 + be_k$. Since $k \neq 1$, ensuring that $e_1$ and $e_k$ are distinct standard basis vectors. Our goal is to verify that for the constructed matrix $A$, the transformation yields $Ay = ae_1$.

The matrix $A$ is given by:
\[
A = E_{1,1} + d_{1}abE_{k,1} - d_{1}a^{2}E_{k,k} + \sum_{i\in[n],i\ne1,i\ne k}E_{i,i}
\]
Let's apply $A$ to $y$ by distributing the terms:
\begin{align*}
    Ay = \left( E_{1,1} + d_{1}abE_{k,1} - d_{1}a^{2}E_{k,k} + \sum_{i\in[n],i\ne1,i\ne k}E_{i,i} \right) (ae_1 + be_k)
\end{align*}
We evaluate the effect of each component of $A$ on $y$, noting that the elementary matrix $E_{i,j}$ acts on a basis vector $e_m$ as $E_{i,j}e_m = \delta_{jm}e_i$, where $\delta_{jm}$ is the Kronecker delta.

\begin{enumerate}
    \item \textbf{First Term:} $E_{1,1}(ae_1 + be_k) = a E_{1,1}e_1 + b E_{1,1}e_k = a e_1 + 0 = \boldsymbol{ae_1}$.

    \item \textbf{Second Term:} $d_{1}abE_{k,1}(ae_1 + be_k) = a d_{1}ab E_{k,1}e_1 + b d_{1}ab E_{k,1}e_k = a d_{1}ab e_k + 0 = \boldsymbol{d_{1}a^2b e_k}$.

    \item \textbf{Third Term:} $-d_{1}a^{2}E_{k,k}(ae_1 + be_k) = -a d_{1}a^{2} E_{k,k}e_1 - b d_{1}a^{2} E_{k,k}e_k = 0 - b d_{1}a^{2} e_k = \boldsymbol{-d_{1}a^2b e_k}$.

    \item \textbf{Fourth Term:} The sum $\sum_{i\in[n],i\ne1,i\ne k}E_{i,i}$ acts as the identity on basis vectors with indices other than $1$ and $k$. Therefore, its product with $(ae_1+be_k)$ is zero.
\end{enumerate}

Summing the results from each component:
\begin{align*}
    Ay &= \boldsymbol{ae_1} + \boldsymbol{d_{1}a^2b e_k} - \boldsymbol{d_{1}a^2b e_k} \\
       &= ae_{1}.
\end{align*}
Thus, we have verified that $Ay=ae_{1}$, as required. 

Conclusion:
We have constructed \( A \in G \) such that \( A y = a e_1 \). Therefore, the lemma is proved.

\end{proof}

\begin{remark}
The Lemma. \ref{lemma:psd} is also applicable when $d_i < 0,\; \forall i \in  \{1, \dots ,k\}$. We just need to appropriately modify the matrices $S$ and $R$.
\end{remark}

\subsection{Canonical Alignment in the Indefinite Case}
\label{sec:canonical-alignment-indef}

\begin{lemma}
\label{lemma:pos-neg}
Let \( D \in \euclidean{n \times n} \) be a diagonal matrix defined as:
\begin{equation}
    D = \sum_{i=1}^{k-1} d_i E_{i,i} - \sum_{i=k}^l d_i E_{i,i},
    \label{eq:diag-pos-neg}
\end{equation}
where \( d_1, \dots, d_l > 0 \), \( k < l \), and \( e_i \) denotes the standard basis vectors in \(\euclidean{n}\). Let \( x = a e_1 + b e_k \in \euclidean{n} \). Let \( G \) be the orthogonal group preserving the quadratic form \( x^T D x \). Then, there exists \( A \in G \) such that:
\begin{equation}
    Ax = 
    \begin{cases}
        \alpha e_1, & x^T D x > 0, \\
        \alpha e_k, & x^T D x < 0,
    \end{cases}
    \label{eq:lemma3-claim}
\end{equation}
where \( \alpha \) is an appropriate scaling factor such that \( x^T D x = x^T A^T D A x \).
\end{lemma}

\begin{proof}
We will prove the result for the case \( x^T D x > 0 \). Similar steps follow for the other case. The value of \( x^T D x \) is:
\begin{equation}
    x^T D x = d_1 a^2 - d_k b^2.
\end{equation}
Since \( d_1 a^2 - d_k b^2 > 0 \), define the scalar \( \beta \) as:
\begin{equation}
    \beta = \frac{1}{\sqrt{d_1 a^2 - d_k b^2}}.
\end{equation}

Define an \( n \times n \) matrix \( S \) as:
\begin{equation}
    S = \sqrt{d_1} E_{1,1} + \sqrt{d_k} E_{k,k} + \sum_{i \in [n], \; i \neq 1,k} E_{i,i}.
    \label{eq:def-S-L3}
\end{equation}
Now define another \( n \times n \) matrix \( R \) as:
\begin{equation}
    R = \beta \paren{\sqrt{d_1} a E_{1,1} - \sqrt{d_k} b E_{1,k} - \sqrt{d_k} b E_{k,1} + \sqrt{d_1} a E_{k,k} + \sum_{i \in [n], \; i \neq 1,k} E_{i,i}}.
\end{equation}
Set \( A = S^{-1} R S \). Our goal is to show that \( A^T D A = D \) and \( A x = \alpha e_1 \).

\paragraph{Step 1: Show that \( A^T D A = D \).}

Consider \( A^T D A \):
\begin{align}
    A^T D A &= \paren{S R^T S^{-1}} D \paren{S^{-1} R S} \nonumber \\
    &= \paren{R S}^T \paren{S^{-1} D S^{-1}} \paren{R S}.
    \label{eq:AtDA-L3}
\end{align}
We note that:
\begin{equation}
    S^{-1} D S^{-1} = E_{1,1} - E_{k,k} + \sum_{i=2}^{k-1} d_i E_{i,i} - \sum_{i=k+1}^l d_i E_{i,i},
    \label{eq:SiDSi-L3}
\end{equation}
and:
\begin{equation}
    R S = \beta \paren{d_1 a E_{1,1} - d_k b E_{1,k} - \sqrt{d_1 d_k} b E_{k,1} + \sqrt{d_1 d_k} a E_{k,k} + \sum_{i \in [n], \; i \neq 1,k} E_{i,i}}.
    \label{eq:RS-L3}
\end{equation}
Substituting \eqref{eq:SiDSi-L3} and \eqref{eq:RS-L3} into \eqref{eq:AtDA-L3}, we get:
\begin{align}
    A^T D A &= \beta^2 \bracket{\frac{1}{\beta^2} d_1 E_{1,1} - \frac{1}{\beta^2} d_k E_{k,k} + \sum_{i=2}^{k-1} d_i E_{i,i} - \sum_{i=k+1}^l d_i E_{i,i}} \nonumber \\
    &= \sum_{i=1}^{k-1} d_i E_{i,i} - \sum_{i=k}^l d_i E_{i,i} \nonumber \\
    &= D.
    \label{eq:Step1-L3}
\end{align}

\paragraph{Step 2: Show that \( A x = \alpha e_1 \).}

From \eqref{eq:RS-L3} and \eqref{eq:def-S-L3}, we compute:
\begin{align}
    A &= S^{-1} R S \nonumber \\
    &= \beta \paren{\sqrt{d_1} a E_{1,1} - \frac{d_k}{\sqrt{d_1}} b E_{1,k} - \sqrt{d_1} b E_{k,1} + \sqrt{d_1} a E_{k,k} + \sum_{i \in [n], \; i \neq 1,k} E_{i,i}}.
\end{align}
Thus, the value of \( A x \) is:
\begin{align}
    A x &= \beta \bracket{\paren{\sqrt{d_1} a^2 - \frac{d_k}{\sqrt{d_1}} b^2} e_1} \nonumber \\
    &= \frac{\beta}{\sqrt{d_1}} \bracket{\paren{d_1 a^2 - d_k b^2} e_1} \nonumber \\
    &= \frac{\beta}{\sqrt{d_1}} \bracket{\frac{1}{\beta^2} e_1} \nonumber \\
    &= \frac{1}{\beta \sqrt{d_1}} e_1.
\end{align}
Hence, \( A x = \alpha e_1 \), where \( \alpha = \frac{1}{\beta \sqrt{d_1}} = \frac{\sqrt{d_1 a^2 - d_k b^2}}{\sqrt{d_1}} \).
\end{proof}

\subsection{Orbit Space and Quadratic-Form Norm Correspondence}
\label{sec:orbit-space-proof}

Building upon Lemmas~\ref{lemma:canonical}, \ref{lemma:psd}, and \ref{lemma:pos-neg}, 
we now establish Theorem~\ref{thm:orbit-space}, which characterizes the structure of 
group orbits under quadratic-form–preserving transformations.
\orbitSpace*

\begin{proof}

We have to show that each set $\mathcal{O}_c$ is a unique orbit. Suppose $x \in \mathcal{O}_c$, then by definition of orbit under the action of given orthogonal group, the orbif of $x$ is subset of  $\mathcal{O}_c$, i.e., $G \cdot x \subset \mathcal{O}_c$. Now, we need to show the subset relation in reverse direction, i.e., $\mathcal{O}_c \subset G \cdot x$.  This is equivalent  to proving : if \( x^T A x = y^T A y \), then \( x \) and \( y \) must lie in the same orbit.

To achieve this, we first diagonalize the symmetric matrix \( A \). Let \( UAU^T = D \), where \( D \) is a diagonal matrix. By selecting an appropriate permutation matrix \( P \), we further transform \( D \) into a canonical form which partitions $D$ into the positive, negative, and zero eigenvalues of \( A \). Without loss of generality, we assume this structure for \( A \) and proceed with the proof.

Let \( x \in \euclidean{n} \). We partition the diagonal matrix \( D \) into positive, negative, and zero components:
\[
D = \bracket{
    \begin{array}{ccc}
        D_p &       & \\
            & D_n    &\\
            &       & D_z
    \end{array}
},
\]
where \( D_z = \boldsymbol{0} \) (zero matrix). Correspondingly, partition \( x \) as:
\[
x = \bracket{
    \begin{array}{c}
        x_p \\ x_n \\ x_z
    \end{array}
},
\]
so that
\begin{align}
    \Vert x \Vert_D^2 &= \paren{x^T D x} \nonumber \\
    &= \paren{x_p^T D_p x_p} + \paren{x_n^T D_n x_n} + \paren{x_z^T D_z x_z} \nonumber \\
    &= \Vert x_p \Vert_{D_p}^2 + \Vert x_n \Vert_{D_n}^2 + \Vert x_z \Vert_{D_z}^2.
\end{align}
Since \( \Vert x_z \Vert_{D_z}^2 = 0 \) because \( D_z = \boldsymbol{0} \), we have \( \paren{x_z^T D_z x_z} = 0 \).

\paragraph{Step 1: Positive and Negative diagonal entries \newline}
Apply Lemma~\ref{lemma:canonical} to \( \paren{D_p, x_p} \) and \( \paren{D_n, x_n} \), yielding matrices \( A_p \) and \( A_n \) such that:
\[
A_p^T D_p A_p = D_p, \quad A_p x_p = \alpha e_1,
\]
and
\[
A_n^T D_n A_n = D_n, \quad A_n x_n = \beta_1 e_1.
\]

\paragraph{Step 2: Zero diagonal entries \newline}
Since \( \Vert x_z \Vert_{D_z}^2 = 0 \), for any \( y \in \euclidean{n-k-l} \), we can find an orthonormal matrix \( A_z \) (i.e., \( A_z^T A_z = I \)) such that:
\[
A_z^T D_z A_z = D_z = \boldsymbol{0}, \quad A_z x_z = \beta_2 e_1.
\]

Define \( A_1 \) as:
\[
A_1 = \bracket{
    \begin{array}{ccc}
        A_p &       & \\
            & A_n    &\\
            &       & A_z
    \end{array}
}.
\]
Since \( A_1^T D A_1 = D \), we have \( A_1 \in G \). Moreover,
\[
A_1 x = \bracket{
    \begin{array}{ccc}
        A_p &       & \\
            & A_n    &\\
            &       & A_z
    \end{array}
} \bracket{
    \begin{array}{c}
        x_p \\ x_n \\ x_z
    \end{array}
}
= \bracket{
    \begin{array}{c}
        A_p x_p \\ A_n x_n \\ A_z x_z
    \end{array}
} = \alpha e_1 + \beta_1 e_{k+1} + \beta_2 e_{l+1}.
\]

\paragraph{Step 3: Combining Negative and Zero Components \newline}
Concatenate \( x_n \) and \( x_z \):
\[
x_{nz} = \beta_1 e_1 + \beta_2 e_{l-k+1}.
\]
Combine \( D_n \) and \( D_z \) as:
\[
D_{nz} = \bracket{
    \begin{array}{cc}
        D_n & \\
            & D_z
    \end{array}
}.
\]
Applying Lemma~\ref{lemma:psd} to \( \paren{D_{nz}, x_{nz}} \), we find \( A_{nz} \) such that:
\[
A_{nz}^T D_{nz} A_{nz} = D_{nz}, \quad A_{nz} x_{nz} = \beta e_1.
\]

Define \( A_2 \) as:
\[
A_2 = \bracket{
    \begin{array}{cc}
        I & \\
          & A_{nz}
    \end{array}
}.
\]
Since \( A_2^T D A_2 = D \), we have \( A_2 \in G \). Now:
\[
A_2 A_1 x = \alpha e_1 + \beta e_{k+1}.
\]

\paragraph{Step 4: Aligning with Canonical Basis \newline}
Applying Lemma~\ref{lemma:pos-neg} to \( \paren{D, A_2 A_1 x} \) (for the case of $x^TDx > 0$), we find \( A_3 \) such that:
\begin{equation}
A_3^T D A_3 = D, \quad A_3 \paren{A_2 A_1 x} = \gamma e_1.
\label{eq:orbit-space-claim-proof-1}
\end{equation}
Similar steps hold in the case of $x^TDx < 0$, i.e., we can find $A_3$ such that $A_3^TDA_3 = D \; \; \And \; A_3\paren{A_2A_1x} = \gamma e_{k+1}$. Since \( A_1, A_2, A_3 \in G \), their product \( A_3 A_2 A_1 \in G \). Let, $W = A_3 A_2 A_1$. Then, we can write,
\begin{equation}
W^T D W = D, \quad Wx = \gamma e_1.
\label{eq:orbit-space-claim-proof-2}
\end{equation}
\paragraph{Conclusion \newline}
Since, the eq.~\eqref{eq:orbit-space-claim-proof-2} holds for any \( x \in \euclidean{n} \) having the same value for \( \Vert x \Vert_D \), we can conclude that \( x, y \in \euclidean{n} \) lie in the same orbit if \( \paren{\Vert x \Vert_D} = \paren{\Vert y \Vert_D} \).

\end{proof}

\subsection{Decomposition into Scale- and Norm-Invariant Components}
\label{sec:decomposition}

Building upon the orbit characterization established in 
Theorem~\ref{thm:orbit-space}, we now derive our main result: 
a canonical decomposition of any \(G\)-equivariant function into 
\emph{scale-invariant} and \emph{norm-invariant} components. 
The orbit-space analysis reveals that the orbit of each input vector 
is uniquely determined by its quadratic-form norm. 
In the following result, we further show that the combination of 
this norm and the corresponding normalized direction 
completely characterizes each input, thereby enabling the 
desired factorization of equivariant functions.

The following theorem formalizes this decomposition result.

\mainResult*
\begin{proof}
Since \( f(x) \) is equivariant under the action of \( G \), we have
\begin{equation}
    f \paren{g \cdot x} = g \cdot f(x), \quad \forall g \in G, x \in \euclidean{n}.
\end{equation}
This can be rewritten in terms of a canonical form as,
\begin{equation}
    f(y) = \phi_1(y) \cdot \phi_2 \paren{r(y)},
    \label{eq:inv-eq}
\end{equation}
where, 
\begin{equation}
y = \phi_1(y) \cdot r(y).
\label{eq:y-eq}
\end{equation}
 Here, \(r(y) \in \euclidean{n} \setminus U  \) is a representative element chosen from the entire orbit of \( y \), denoted by \( G \cdot y \). The function \( \phi_1 : \euclidean{n} \setminus U \rightarrow G \) maps each point \( y \) to an element of \( G \) that aligns \( y \) with its orbit representative \( r(y) \).

Since, \( r(y) \) is the same for every element in the orbit of \( y \), it follows that \( \phi_2 \circ r \) is \( G \)-invariant. Therefore, applying Theorem \ref{thm:orbit-space}, we can express this as

\begin{equation}
    \left( \phi_2 \circ r \right)(x) = \phi_n\left( \Vert x \Vert_A \right),
    \label{eq:norm-inv}
\end{equation}
for some function \( \phi_n \).

Next, we select a representative element \( r(y) \) such that it satisfies the homogeneity condition \( r(\alpha y) = \alpha \, r(y) \) for any scalar \( \alpha \in \euclidean{} \) with \( \alpha > 0 \). This can be achieved by setting \( r(y) = \alpha e_i \), where \( \alpha = \| y \|_A \) represents the norm of \( y \) with respect to the matrix \( A \), and the choice of the standard basis vector \( e_i \) depends on the sign of \( \| y \|_A \), denoted by \( \operatorname{sign} \left( \| y \|_A \right) \).

This specific selection of \( r(y) \) implies that if we define,
\begin{equation}
    \phi_1(\alpha y) = \phi_1(y),
    \label{eq:phi_1-choice}
\end{equation}
then the following two equations, derived from Eq.~\eqref{eq:y-eq} and the chosen form of \( r(y) \), will hold consistently:
\begin{align}
    \alpha y = \alpha \paren{\phi_1(y) \cdot r(y)} &= \phi_1(y) \cdot \alpha r(y), \\
    \alpha y &= \phi_1 \paren{(\alpha y)} \cdot \alpha r(y).
\end{align}

Thus, the eq.~\eqref{eq:phi_1-choice} shows that \( \phi_1 \) is scale-invariant. Hence, we can write,
\begin{equation}
    \phi_1(x) = \phi_s\left( \frac{x}{\Vert x \Vert_A} \right),
    \label{eq:scale-inv}
\end{equation}
for some function \( \phi_s \). 

Finally, substituting eq.~\eqref{eq:norm-inv} and \eqref{eq:scale-inv} into eq.~\eqref{eq:inv-eq}, we obtain,

\[
f(x) = \phi_s\paren{\frac{x}{\Vert x \Vert_A}} \cdot \phi_n\paren{\Vert x \Vert_A}.
\]
\end{proof}

\subsection{Regularity conditions}
\label{sec:reg-cond}
\regularity*


\begin{proof}
The proof relies on analyzing the regularity of the components in the rearranged decomposition from Theorem 4.2.

\begin{enumerate}
    \item From the proof of Theorem~\ref{thm:main-result}, we can isolate the norm-invariant component $\phi_n$ by applying the inverse of the group element returned by $\phi_s$.
    \begin{equation}
        \phi_n(\|x\|_A) = \left(\phi_s\left(\frac{x}{\|x\|_A}\right)\right)^{-1} f(x) \label{eq:rearranged}
    \end{equation}
    This equation shows that the regularity of $\phi_n$ depends on the regularity of $f$ and the map $x \mapsto \phi_s\left(\frac{x}{\|x\|_A}\right)$. Since $\phi_s$ maps to the Lie group $G$, and both group action and group inversion are smooth operations, the regularity of the expression is determined by the map $x \mapsto \frac{x}{\|x\|_A}$ and the function $f(x)$ itself.

    \item \textbf{Case 1: A is Positive or Negative Definite} \\
    If matrix $A$ is positive or negative definite, the quadratic form $x^T A x$ is non-zero for all non-zero $x \in \mathbb{R}^n$. This means the set $U := \{x \in \mathbb{R}^n : x^T A x = 0\}$ contains only the zero vector. The domain of analysis is $\mathbb{R}^n \setminus U = \mathbb{R}^n \setminus \{0\}$.

    The norm is defined as $\|x\|_A = \text{sign}(x^T A x) \sqrt{|x^T A x|}$. When $A$ is definite, $\text{sign}(x^T A x)$ is constant for all $x \neq 0$. The expression inside the square root, $|x^T A x|$, is a smooth and strictly positive polynomial for $x \neq 0$. The square root of a smooth, strictly positive function is smooth.

    Therefore, the map $x \mapsto \|x\|_A$ is smooth on $\mathbb{R}^n \setminus \{0\}$. Consequently, the map $x \mapsto \frac{x}{\|x\|_A}$ is a composition of smooth functions and is itself smooth[cite: 622]. Since the group action is smooth, the function $x \mapsto \left(\phi_s\left(\frac{x}{\|x\|_A}\right)\right)^{-1}$ is also smooth.

    If we assume $f$ is smooth ($C^{\infty}$), then the right-hand side of Eq. \eqref{eq:rearranged} is a product of smooth functions, which is smooth. Thus, $\phi_n$ must also be smooth[cite: 624].

    \item \textbf{Case 2: A is Indefinite} \\
    If $A$ is neither positive nor negative definite, the set $U$ contains non-zero vectors. The function $g(z) = \sqrt{|z|}$ is continuous everywhere but not differentiable at $z=0$. Because $x^T A x$ can be zero for non-zero $x$, the map $x \mapsto \|x\|_A$ is continuous on $\mathbb{R}^n$ but is not differentiable on the set $U$. The map $x \mapsto \frac{x}{\|x\|_A}$ is therefore continuous on its domain $\mathbb{R}^n \setminus U$.

    If we assume $f$ is continuous ($C^0$), then the right-hand side of Eq. \eqref{eq:rearranged}, being a product and composition of continuous functions, is continuous. It follows that $\phi_n$ must also be continuous.
\end{enumerate}
This covers both cases and completes the proof.
\end{proof}

\subsection{Extension to Diagonal action}
\label{sec:diag-action-tuples}
\begin{theorem}[Classification of $p$-tuples up to the $A$-orthogonal group, cf.\ \protect{\cite[Chapter~6]{Jacobson-BasicAlgebra1}}]
\label{thm:classification-ptuples-isometry-A}
Let $A$ be a non-degenerate real symmetric $n \times n$ matrix, and consider 
the bilinear form $\langle x, y\rangle_A = x^\top A\, y$ on $\mathbb{R}^n$. 
Suppose $(x_1,\dots,x_p)$ and $(y_1,\dots,y_p)$ are two ordered $p$-tuples 
in $\mathbb{R}^n$. Then the following are equivalent:
\begin{enumerate}
  \item They have the same $A$-Gram matrix, i.e.
  \[
    x_i^\top A\, x_j \;=\; y_i^\top A\, y_j 
    \quad\text{for all } 1 \le i,j \le p.
  \]
  \item There exists an invertible linear map $Q\in \mathrm{GL}(n,\mathbb{R})$ 
    such that
    \[
      Q^\top A\, Q \;=\; A
      \quad\text{and}\quad
      Q\,x_i \;=\; y_i \quad \text{for all } i.
    \]
\end{enumerate}
Equivalently, $(x_1,\dots,x_p)$ and $(y_1,\dots,y_p)$ lie in the same orbit 
under the action of the group 
\[
  G \;=\; \{\,Q\in \mathrm{GL}(n,\mathbb{R}) : Q^\top A\,Q = A \}.
\]
\end{theorem}

\ExtendedDecomp*

\begin{proof}
In the proof of Theorem~\ref{thm:main-result}, we utilized the orbit correspondence result (Theorem~\ref{thm:orbit-space}) to establish the norm invariance part ($\phi_n$) of the function decomposition. In the setting of the diagonal action on tuples, the  invariance part ($\phi_n$) of the decomposition follows immediately from Theorem~\ref{thm:classification-ptuples-isometry-A}. The remainder of the proof proceeds analogously to that of Theorem~\ref{thm:main-result}.
\end{proof}

\subsection{Householder reflection}
\label{sec:Householder-reflection}
\propSymmForm*

\begin{proof}
The proof combines the canonical decomposition from Theorem~\ref{thm:main-result} with the specific properties of the group action and the scale-invariant component.

\begin{enumerate}
    \item \textbf{Canonical Decomposition:} According to Theorem~\ref{thm:main-result}, any G-equivariant function $f: X \to Y$ can be written as the action of a scale-invariant component on a norm-invariant component:
    \begin{equation}
        f(x) = \phi_s\left(\frac{x}{\|x\|_A}\right) \cdot \phi_n(\|x\|_A)
        \label{eq:decomp}
    \end{equation}
    where $\phi_s: X \to G$ and $\phi_n: \mathbb{R} \to Y$. The symbol `$\cdot$` denotes the group action of $G$ on the output space $Y$.

    \item \textbf{Group Action as Conjugation:} The proposition states that the group action on the output is conjugation. For a group element $g \in G$ and an output object $Y \in Y$, this action is defined as $g \cdot Y = gYg^{-1}$. Applying this definition to our decomposition in Eq. \eqref{eq:decomp}, the action of the group element $\phi_s(\dots)$ on the output $\phi_n(\dots)$ is:
    \begin{equation}
        f(x) = \phi_s\left(\frac{x}{\|x\|_A}\right) \phi_n(\|x\|_A) \left(\phi_s\left(\frac{x}{\|x\|_A}\right)\right)^{-1}
        \label{eq:conjugation_supp}
    \end{equation}

    \item \textbf{Householder Reflection Property:} As discussed in Section 5.2, the scale-invariant component $\phi_s(\frac{x}{\|x\|_A})$ is equivalent to a Householder reflection. A key property of a Householder reflection matrix $R$ is that it is its own inverse, i.e., $R = R^{-1}$. Therefore, we have:
    \begin{equation}
        \phi_s\left(\frac{x}{\|x\|_A}\right) = \left(\phi_s\left(\frac{x}{\|x\|_A}\right)\right)^{-1}
        \label{eq:householder_supp}
    \end{equation}

    \item \textbf{Final Form:} By substituting the property from Eq. \eqref{eq:householder_supp} into Eq. \eqref{eq:conjugation_supp}, we replace the inverse term with the term itself:
    \begin{align*}
        f(x) &= \phi_s\left(\frac{x}{\|x\|_A}\right) \phi_n(\|x\|_A) \left(\phi_s\left(\frac{x}{\|x\|_A}\right)\right)^{-1} \\
             &= \phi_s\left(\frac{x}{\|x\|_A}\right) \phi_n(\|x\|_A) \phi_s\left(\frac{x}{\|x\|_A}\right)
    \end{align*}
\end{enumerate}
This yields the symmetric form of the function as stated in the proposition.
\end{proof}

\end{document}